\documentclass{article}

\PassOptionsToPackage{numbers, sort&compress}{natbib} 

\usepackage[preprint]{neurips_2026}

\usepackage[utf8]{inputenc} 
\usepackage[T1]{fontenc}    
\usepackage{hyperref}       
\usepackage{url}            
\usepackage{booktabs}       
\usepackage{amsfonts}
\usepackage{amsmath} 
\usepackage{nicefrac}       
\usepackage{microtype}      
\usepackage{xcolor}         
\usepackage{pifont}
\usepackage{graphicx}
\usepackage{multicol}
\usepackage{multirow}
\usepackage{newfloat}
\usepackage{arydshln}
\usepackage{amsthm}
\usepackage{amssymb}
\usepackage{tikz}
\usepackage{todonotes}

\newsavebox{\normalnum}
\newsavebox{\boldnum}
\newcommand{\bfnum}[1]{%
  \sbox{\normalnum}{#1}%
  \sbox{\boldnum}{\textbf{#1}}%
  \raisebox{0pt}[\ht\normalnum][\dp\normalnum]{%
    \resizebox{\wd\normalnum}{!}{\usebox{\boldnum}}%
  }%
}

\usetikzlibrary{arrows.meta,positioning,fit,calc,decorations.pathreplacing}

\tikzset{
  >=Latex,
  io/.style={draw, rounded corners=2pt, align=center, minimum width=3.0cm, minimum height=0.95cm, inner sep=4pt},
  box/.style={draw, rounded corners=2pt, align=center, minimum width=4.3cm, minimum height=0.95cm, inner sep=4pt},
  pbox/.style={draw, rounded corners=2pt, align=center, minimum width=3.7cm, minimum height=0.95cm, inner sep=4pt},
  sum/.style={draw, circle, minimum size=8mm, inner sep=0pt, font=\large},
  note/.style={draw, rounded corners=3pt, align=left, inner sep=5pt, font=\footnotesize, text width=11.5cm},
  conn/.style={-Latex, line width=0.9pt},
  lab/.style={font=\footnotesize\bfseries, draw, rounded corners=2pt, fill=white, inner sep=2pt},
  klab/.style={font=\scriptsize}
}

\definecolor{blueA}{RGB}{38,84,124}
\definecolor{blueB}{RGB}{90,140,191}
\definecolor{blueC}{RGB}{220,233,246}
\definecolor{grayA}{RGB}{84,92,100}
\definecolor{grayB}{RGB}{235,238,242}
\definecolor{redA}{RGB}{179,68,68}
\definecolor{goldA}{RGB}{214,154,56}

\tikzset{
  >=Latex,
  panel/.style={draw=grayA!70, rounded corners=3pt, line width=0.6pt, fill=white},
  paneltitle/.style={font=\bfseries\footnotesize, anchor=west},
  subtitle/.style={font=\scriptsize, text=grayA},
  flow/.style={-Latex, line width=0.9pt, draw=blueA},
  block/.style={draw=blueA!80, rounded corners=2pt, line width=0.8pt, fill=blueC, align=center, inner sep=4pt},
  smallblock/.style={draw=blueA!80, rounded corners=2pt, line width=0.7pt, fill=blueC!65, align=center, inner sep=3pt},
  badge/.style={draw=grayA!70, rounded corners=2pt, line width=0.6pt, fill=grayB, align=center, inner sep=4pt},
  sumdot/.style={draw=blueA!85, circle, fill=white, line width=0.8pt, minimum size=7mm, inner sep=0pt},
  obsdot/.style={circle, fill=black, inner sep=0pt, minimum size=1.6pt},
  gtbox/.style={draw=grayA!70, rounded corners=2pt, fill=grayB!45},
  atomlabel/.style={font=\scriptsize\bfseries, text=blueA},
  tinylabel/.style={font=\scriptsize, text=grayA},
}

\newtheorem{theorem}{Theorem}

\newtheorem{lemma}[theorem]{Lemma}
\theoremstyle{definition}

\theoremstyle{remark}

\newcommand{\xmark}{\ding{55}}
\newcommand{\gt}{{\mathrm{gt}}}
\newcommand{\B}{\boldsymbol{B}}
\newcommand{\E}{\boldsymbol{E}}
\newcommand{\x}{\boldsymbol{x}}
\newcommand{\z}{\boldsymbol{z}}
\newcommand{\Z}{\boldsymbol{Z}}
\newcommand{\w}{\boldsymbol{w}}
\newcommand{\p}{\boldsymbol{p}}
\newcommand{\n}{\boldsymbol{n}}

\newcommand{\R}{\mathbb{R}}

\def\methodname{\textit{FLASH-MAX}}

\title{Fast Reconstruction of Exact Maxwell Dynamics \\from Sparse Data}

\author{%
  Dan DeGenaro\textsuperscript{1} \\
  \And
  Xin Li\textsuperscript{2} \\
  \And
  Obed Amo\textsuperscript{3} \\
  \AND
  Michael Pokojovy\textsuperscript{3,4} \\
  \And
  Sarah Adel Bargal\textsuperscript{1} \\
  \And
  Markus Lange-Hegermann\textsuperscript{5} \\
  \And
  Bogdan Rai\c{t}\u{a}\textsuperscript{2}$^*$ \\
  [0.5em]
  \textsuperscript{1}Department of Computer Science, Georgetown University \\
  \texttt{\{drd92, sarah.bargal\}@georgetown.edu} \\
  \textsuperscript{2}Department of Mathematics, Georgetown University \\
  \texttt{\{xl572, bogdan.raita\}@georgetown.edu} \\
  \textsuperscript{3}Department of Mathematics and Statistics, Old Dominion University \\
  \texttt{\{oamo, mpokojovy\}@odu.edu} \\
  \textsuperscript{4}School of Data Science, Old Dominion University \\
  \textsuperscript{5}Institute Industrial IT, Department of Computer Science and Automation,\\
  OWL University of Applied Sciences and Arts \\
  \texttt{markus.lange-hegermann@th-owl.de}
}

\begin{document}

\maketitle

\begin{abstract}
We introduce \methodname{}, a shallow, exact-by-construction neural network architecture for predicting homogeneous electromagnetic fields from sparse pointwise observations. Each hidden neuron represents a separate exact solution
to Maxwell's equations, so that the network satisfies the governing equations 
symbolically by construction and can be trained end-to-end from sparse data within seconds. We prove a universal approximation
result showing that this exact model class remains universal on arbitrary domains.
\methodname{} reaches sub-1\% relative validation error from about 1K
sparse pointwise observations in seconds, all while maintaining a zero PDE residual, and keeps single-digit errors even for only 100 observations sampled from 3D space.
These results suggest that moving governing structure from the loss into the hypothesis class can dramatically improve the trade-off between precision and optimization speed in scientific machine learning. 
\end{abstract}

%

\section{Introduction}
\label{Introduction}

Maxwell's equations are a system of eight partial differential equations (PDEs) that describe the electromagnetic interaction in classical physics. Because Maxwell’s equations govern electromagnetic wave propagation, reconstructing a full spatiotemporal field from sparse pointwise observations is a basic task in applications such as antenna measurement, electromagnetic compatibility diagnostics, RF sensing, subsurface exploration, tomography, and fusion plasma diagnostics. In such settings, the available data do not define a fully specified forward problem, so the task is not standard forward simulation but rather  reconstruction of fields from sparse observations. We build the structure of Maxwell's equations directly into our model architecture and show that this can make learning both more efficient and more accurate, as opposed to enforcing admissibility approximately through the loss. In the sparse-reconstruction benchmarks considered here, the resulting training procedure reaches the target validation accuracy orders of magnitude faster than the reference implementations.

We propose \methodname{}, a Fast-Learning Accurate Solver for the Homogeneous MAXwell equations. \methodname{} is a \emph{shallow neural network whose hidden neurons are themselves exact solutions of Maxwell's equations}, enabled by a suitable weight-sharing parameterization. This ensures that the field predicted by this network remains consistent with Maxwell's equations throughout training and inference. Hence, \methodname{} builds the governing structure of Maxwell's equations directly into the model architecture rather than relying on the loss to enforce approximate physical validity.

Among learning-based approaches, Physics-Informed Neural Networks (PINNs) are the closest comparison family, since they also fit PDE-constrained models directly from data while enforcing governing equations during optimization~\cite{raissi2019physics}. Ehrenpreis-Palamodov Gaussian Processes (EPGPs) are also conceptually relevant because they embed field structure directly into a Gaussian Process covariance model~\cite{harkonen2023gaussian}. Numerical solvers and neural operators both assume a fully specified/well-posed problem and are hence not well-suited for the sparse-reconstruction setting studied here.

The exact model class remains expressive for the homogeneous system: Theorem~\ref{thm:spacetimeUAT} shows that it is universal on arbitrary domains. So exact hidden-unit admissibility need not force the model into an overly rigid approximation class. The homogeneous Maxwell setting therefore gives a mathematically explicit regime in which exact solution-class parameterizations remain expressive and support very rapid sparse-data reconstruction of electromagnetic fields.

\begin{figure}
\includegraphics[width=\linewidth]{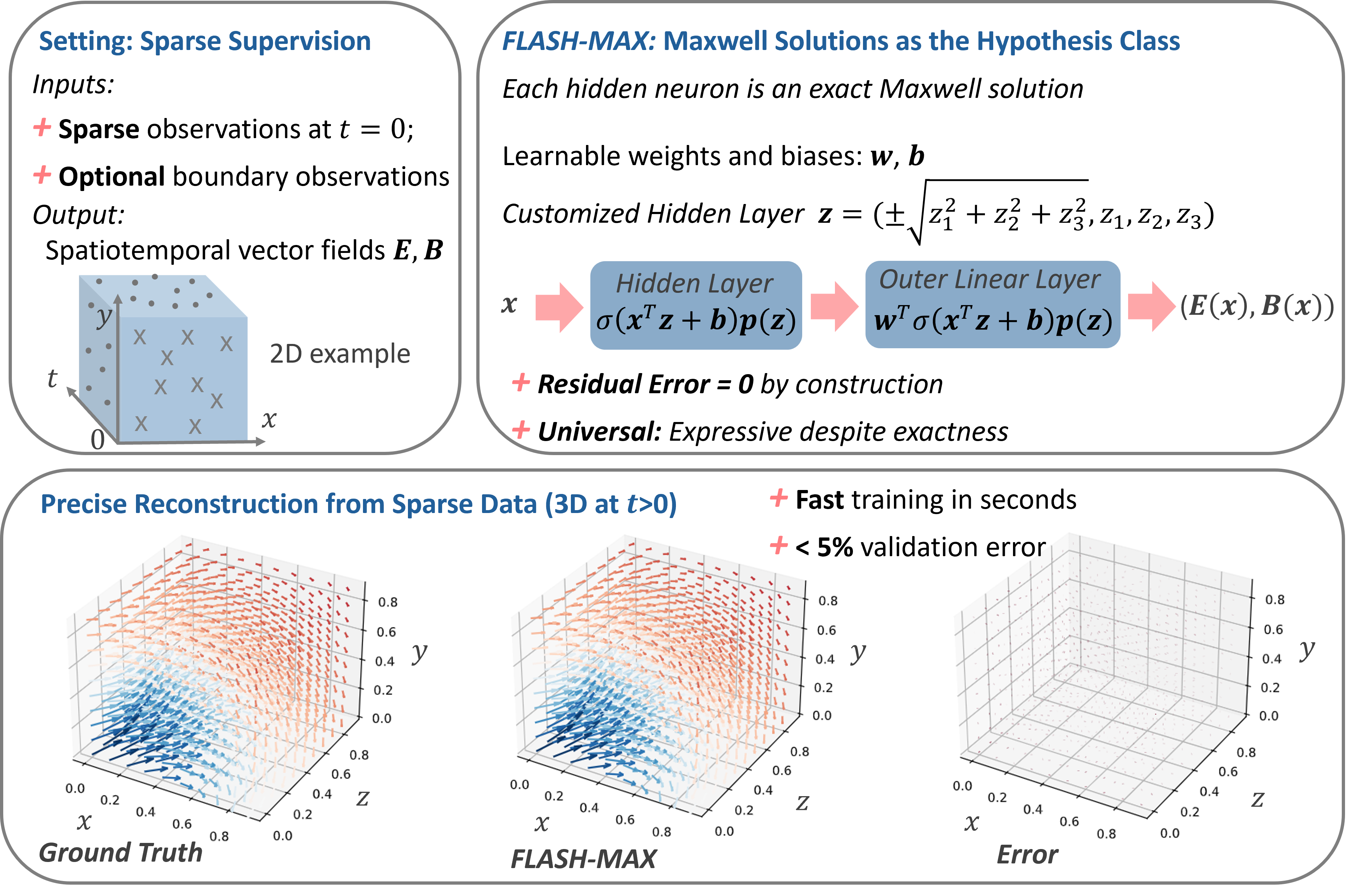}
 \caption{\textbf{Sparse-data Maxwell Reconstruction with Exact Hidden-unit Structure.} From sparse pointwise observations of a single field instance, \methodname{} fits a shallow trainable architecture whose hidden neurons are exact Maxwell solutions, so the predicted field solves Maxwell's system throughout optimization. The resulting model is exact by construction, expressive on arbitrary domains, and yields accurate reconstruction from sparse data. Exactness is ensured by using computer algebra to obtain the hidden layer restriction $z_0^2=z_1^2+z_2^2+z_3^2$ and the polynomials $\boldsymbol p(\boldsymbol{z})$ which tie the vectorial weights. Top-left: Sparse data sampled only at $t=0$ (``$\times$'' symbols), showing that our model is predictive; data can also be sampled on the boundary of the validation domain (``$\bullet$'' symbols). Bottom: Quiver plot visualization of our fast prediction of an electromagnetic knot (Hopfion) at time $t=0.5$. The difference to ground truth is indistinguishable.
 } 
    \label{fig:figure1}
\end{figure}

Empirically, \methodname{} attains sub-1\% relative validation error from about 1K sparse observations in a matter of seconds on representative and diverse sparse-reconstruction tasks while maintaining a zero PDE residual by construction (Table~\ref{tab:race_to_5percent} and Figure~\ref{fig:timing_bc}). Compared with residual-based PINN training, this threshold is reached substantially sooner and more robustly. These results suggest that exact solution-class parameterizations can change the trade-off between physical validity, expressive power, and optimization speed. 
We summarize our contributions as follows:
\begin{itemize}
    \item An exact-by-construction shallow neural architecture for reconstruction of solutions to the homogeneous Maxwell equations from sparse observations, with hidden units that are themselves exact solutions to Maxwell's equations (Figure~\ref{fig:figure1} and Section~\ref{method});
    \item A universality theorem in general domains for this exact model class (Theorem~\ref{thm:spacetimeUAT});
    \item Empirical evidence of fast sparse-data reconstruction with sub-1\% error and  zero PDE residual from 1K data points, and single-digit error from only 100 data points (Section~\ref{experiments});
    \item Empirical and conceptual evidence that moving governing structure from the loss into the hypothesis class can improve the trade-off between physical admissibility and time to useful accuracy when the admissible geometry is explicit.
\end{itemize}

\section{Method}
\label{method}

We first motivate the idea behind \methodname{} as it might be applied to a much simpler linear PDE, the homogeneous transport equation in one spatial dimension.
If we write $f(t,x)=\sigma(z_0t+z_1x+b)$ for a typical neuron in a 2D shallow neural network, then the transport equation yields
\begin{align*}
  0= \partial_tf(t,x)-\partial_xf(t,x)
    &=z_0\cdot \sigma'(z_0t+z_1x+b)-z_1\cdot \sigma'(z_0t+z_1x+b)\\
    &=(z_0-z_1)\cdot \sigma'(z_0t+z_1x+b).
\end{align*}
Hence, $f(t,x)=\sigma(z_0t+z_1x+b)$ is a solution whenever $z_0=z_1$. Thus, in this simple case, the PDE can be enforced exactly by imposing an algebraic constraint on the weights in the activation $\sigma$.

We now extend this idea in several directions, such that it is applicable to electromagnetic phenomena described by Maxwell's equations.

\begin{enumerate}
    \item The domain is four-dimensional, hence we consider $\sigma(z_0t+z_1x+z_2y+z_3z+b)$.
    \item\label{itm:char_variety} Instead of the linear relation between  $z_0$ and $z_1$, we get the quadratic relation\footnote{This
relation can also be explained via the connection between Maxwell's equations
and the wave equation; equivalently, it is the characteristic variety of both the (1+3)D wave equation and the
Maxwell system.}\\ $z_0^2=z_1^2+z_2^2+z_3^2$, which we build into our architecture as a weight-sharing constraint.
    \item We use the linearity of the PDE system to write $f$ as a linear combination: \\ $\sum_kw_k\sigma(z_{k0}t+z_{k1}x+z_{k2}y+z_{k3}z+b_k)$.
    \item To describe electric and magnetic \textit{vector} fields, we provide two Noetherian multipliers $\p_i$, $i\in\{1,2\}$, one of which is introduced into each summand of the linear combination.
    \item We view this as a shallow neural network, where we train the parameters $w_k,\z_k,b_k$ (subject to the relation from item \ref{itm:char_variety} above) on data.
\end{enumerate}


More formally, let $\x=(t,x,y,z)\in\R^4$ be the coordinates of spacetime. We consider, respectively, electric and magnetic fields $\E,\B:\R^4\to\R^3$, with $\B$ sometimes denoted by $\boldsymbol{H}$ in the literature, satisfying the homogeneous Maxwell equations
\begin{equation}\tag{MAX}\label{eq:maxwell}
    \partial_t \E-\nabla \times \B = \boldsymbol0,
    \quad  
    \partial_t\B+ \nabla\times \E = \boldsymbol0,\quad\nabla \cdot \E = 0,\quad
     \nabla\cdot \B = 0.
\end{equation}

We construct an efficient solver for Initial and Initial-Boundary Value Problems using the ansatz:
\begin{align}\tag{\methodname}\label{eq:model}
    (\E(\x),\B(\x))=\sum_{i=1}^2\sum_{j=\pm1}\sum_{k=1}^W w_{ijk}\sigma(\x^\top \z_{ijk}+ b_{ijk})\p_i(\z_{ijk})
\end{align}
which is a shallow neural network of width $W$, customized as follows:

\begin{enumerate}
    \item In the frequency vectors $\z_{ijk}=(z_{ijk0},z_{ijk1},z_{ijk2},z_{ijk3})\in \mathbb{R}^4$, the time weight $z_{ijk0}$ of the time variable is tied to the spatial weights by the algebraic formula 
    \begin{equation}\label{eq:char_var}
        z_{ijk0}=j\sqrt{z_{ijk1}^2+z_{ijk2}^2+z_{ijk3}^2}\quad\text{where $j=\pm1$}. 
    \end{equation}

    \item The $\p_{i}\colon \R^4 \to \R^6$ are Noetherian multipliers~\cite{CCHKL,oberst1999construction,damiano2007computational,nabeshima2022effective} indexed by $i=1,2$ in~\eqref{eq:model}, computed via the computer algebra system Macaulay2~\cite{macaulay2,manssour21linear,cidruiz2021primary,chen22primary,homs21primary,chen2022noetherian,cidruiz2021noetherian} as
    \begin{align}\label{eq:noetherian_multipliers}
       \begin{split}
            &\p_1(\z)=(- z_1z_3, -z_2z_3,z_0^2-z_3^2,-z_0z_2,z_0z_1,{0})\\
            &\p_2(\z)=(z_1z_2,-z_0^2+z_2^2,z_2z_3,-z_0z_3,{0},z_0z_1).
        \end{split}
    \end{align}

    \item The real numbers $b_{ijk}$ are biases and the activation function $\sigma:\R\to\R$ is Lipschitz continuous and not a polynomial. In our experiments, we use the activation function $\sigma=\tanh$ and view $z_{ijk1}$, $z_{ijk2}$, $z_{ijk3}$, ${w}_{ijk}$, and $b_{ijk}$ as learnable real parameters for all $i,j,k$ that can be trained by stochastic gradient descent methods. 
\end{enumerate}

Since the ansatz in~\eqref{eq:model} is obtained by means of computational algebra relying on the Ehrenpreis--Palamodov theorem \cite{EHRENPREIS,PALAMODOV,HORMANDER,BJORK}, we can infer that
every expression~\eqref{eq:model} satisfies Maxwell's equations~\eqref{eq:maxwell}, see Lemma~\ref{lem:solution} in Appendix~\ref{sec:proof_uat}. It is also worth noting that \methodname{} only requires a few hyperparameters. For additional details on \methodname, see Appendix~\ref{app:detailed_flashmax_description}.

A natural concern with hard architectural constraints is the loss of expressive power; if every neuron must satisfy Maxwell's equations exactly, the resulting hypothesis class could be too rigid for sparse-data fitting. The next theorem addresses this concern by proving that every solution of the homogeneous Maxwell equations can be approximated locally by \methodname.

\begin{theorem}[Universal approximation for Maxwell's equations]
\label{thm:spacetimeUAT}
Let $\sigma\colon\R\to\R$ be locally Lipschitz and non-polynomial, $t_1<t_2$, and $\Omega\subset\R^3$ be a bounded open set.
Let $(\E_\gt,\B_\gt)$ be an $L^2$-solution of the homogeneous Maxwell equations~\eqref{eq:maxwell} on $\R^{4}$.
Then, for every $\varepsilon>0$, there exists a finite Maxwell network solution $(\E,\B)$ of the form~\eqref{eq:model} such that
\[
\|( \E_\gt, \B_\gt)-(\E,\B)\|_{L^2((t_1,t_2)\times\Omega)}<\varepsilon.
\]
\end{theorem}
The approximation holds true with mild restrictions on $\sigma$ even after replacing the space $L^2$ with spaces of continuous and smooth functions ($C^0,\,C^m,\,C^\infty)$ or Sobolev spaces $H^m$ for $m>0$. This theorem and its extensions are proved and explored in detail in Appendix~\ref{sec:proof_uat}.
The proof shows that \methodname{} can approximate the ground truth $(\E_\gt(t_1,\cdot),\B_\gt(t_1,\cdot))$ at initial time $t=t_1$ with arbitrarily small $L^2$-norm. It further establishes that the error at $t=t_1$ controls the spacetime $L^2$-error in the time span $(t_1,t_2)$ as a consequence of the finite speed of propagation. For solving initial value problems this implies that reaching small $L^2$-error in \textit{training} at the single time $t=t_1$ implies achieving $L^2$-error for \textit{validation} in the time span $(t_1,t_2)$ of the same order of magnitude.

\section{Related Work}
\label{RelatedWork}


From an ML perspective, \methodname{} is related both to coordinate-based implicit neural representations such as SIREN and Fourier-feature networks~\cite{sitzmann2020siren,tancik2020fourier} and to random-feature or extreme-learning-machine models~\cite{rahimi2007random,pao1994rvfl,huang2006extreme}; unlike both, the trainable features of \methodname{} are constrained to lie in the  solution space of Maxwell's equations.


\textbf{Physics-Informed Neural Networks (PINNs)} solve PDEs mesh-free by embedding governing equations, initial and boundary conditions directly into the loss function. Since the original paper \cite{raissi2019physics}, many enhancements have been proposed to address spectral bias, optimization instability, and slow convergence, including Fourier-based embeddings, strict boundary enforcement, adaptive temporal weighting, domain decomposition, and improved optimization strategies \cite{rahaman2019spectral,tancik2020fourier,wang2021gradient,krishnapriyan2021failure,zhang2021maxwell,roy2024exact,piao2024domain,guo2024tcas}.
Prior PINN electromagnetic studies can be broadly grouped into three classes. First, several works focus on steady-state, frequency-domain, or inverse Maxwell problems, such as magnetostatics, photonic inverse design, and parameter retrieval of electromagnetic material properties \cite{zhang2025electromagnetic,chen2022physics,deng2025physics}. Second, \cite{nohra2024physics,nohra2024approximating,piao2024domain} address heterogeneous media, where interface conditions and material jumps are critical, using strategies such as first-order Maxwell formulations, interface-aware inputs, and domain decomposition. Third, more recent studies  investigate unsteady Maxwell equations, showing that convergence-enhancing strategies such as random Fourier features, strict periodicity enforcement, and causality-aware training can substantially affect performance (\cite{zhang2021maxwell, shaviner2025pinns}).

Beyond classical PINNs, QPINN introduces a hybrid quantum-classical PINN for two-dimensional time-dependent Maxwell equations with additional physical constraints~\cite{chen2026quantum}, while Evo-PINN explores evolutionary and gradient-free optimization to improve PINN training robustness~\cite{11353124}. Despite these advances, existing work remains largely focused on 1D or 2D settings, steady-state or inverse formulations, or specially structured heterogeneous-media problems.
We therefore construct a strong in-house PINN baseline for \eqref{eq:maxwell}, augmenting the original PINN framework~\cite{raissi2019physics} with additional physics-informed constraints such as divergence-free regularization and boundary-condition enforcement. We also compare against adapted QPINN and Evo-PINN baselines in Appendix~\ref{app:pinn_baseline}; since the in-house PINN performs strongest, it is used as the primary PINN comparator in the main text.

Unlike residual-based PINNs, \methodname{} builds Maxwell structure directly into the hypothesis class: every output satisfies the homogeneous Maxwell system by construction, and training only fits sparse observations within this exact class. Thus, its modeled PDE residual vanishes symbolically, whereas a PINN residual is minimized approximately. In all reported examples, \methodname{} reaches the target validation error at least one order of magnitude faster than the PINN baseline.

\textbf{Numerical Methods.}
The Finite Element Method (FEM) for linear Maxwell equations has evolved well beyond standard $H(\mathrm{curl})$-conforming N\'ed\'elec elements~\cite{nedelec1980mixed}, addressing challenges such as spurious modes, material discontinuities, and high-frequency pollution. Recent developments include Discontinuous Galerkin frameworks, in particular IPDG and HDG methods~\cite{nguyen2011hdg}, which provide flexible mesh refinement and local conservation, as well as Finite Element Exterior Calculus (FEEC)~\cite{arnold2006finite}, which preserves the de Rham structure and the associated divergence constraints. For high-frequency time-harmonic Maxwell equations, preconditioners and multiscale methods~\cite{engquist2011sweeping} improve iterative solver efficiency, while in the time domain, structure-preserving integrators combined with adaptive mesh refinement~\cite{bangerth2003adaptive} support high-fidelity electromagnetic interference and scattering simulations.
Compared to alternative techniques like the Finite-Difference Time-Domain (FDTD) method, which struggles with staircasing errors on curved boundaries, or the Boundary Element Method (BEM), which is often limited to piecewise homogeneous media, FEM offers superior flexibility in handling complex, multi-material geometries and anisotropic media through unstructured tetrahedral meshes. This geometric fidelity and availability of automated libraries make FEM the preferred choice for advanced engineering applications, such as stealth technology and photonic crystal design. 

The FEniCS Project is one of the libraries frequently used for high-performance implementation of FEM. Designed as an automated programming environment, it utilizes symbolic domain-specific languages to streamline the generation of efficient $H(\mathrm{curl})$-conforming finite element kernels~\cite{logg2012fenics}. In typical industrial high-frequency applications, mesh sizes can reach up to a billion unknowns, necessitating the use of massively parallel solvers and large computational time on high-performance computing clusters. Unlike traditional FEMs running on HPC, \methodname{} runs on a single L40S GPU in seconds. FEMs at best approach zero PDE residuals asymptotically, whereas \methodname{} has null PDE residuals intrinsically. To separate the architectural speedup from GPU acceleration, Appendix~\ref{app:CPU} reports CPU timings for \methodname{} under the same sparse-reconstruction setup.

\textbf{Neural Operators}
learn mappings between function spaces rather than finite-dimensional parameter vectors and have become a standard paradigm for amortized PDE surrogate modeling. Canonical examples include DeepONet~\cite{lu2021deeponet}, FNO~\cite{li2021fno}, and physics-informed variants such as PINO~\cite{li2024pino}; recent reviews emphasize their ability to generalize across discretizations, but also their dependence on the geometry, representation, and physics of the target problem~\cite{azizzadenesheli2024neuraloperators}.
For electromagnetics, operator-learning approaches exist for certain regimes: frequency-domain photonic-device simulation~\cite{gu2022neurolight,zhu2024pace}, free-form 3D electromagnetic scattering and inverse design~\cite{augenstein2023surrogate}, 2D transient wave propagation~\cite{oh2026pideeponet}, time-domain photonic rollouts~\cite{wu2026dafno}, and Fourier-space hard-constrained electrodynamics tailored to beam-driven settings~\cite{leon2024fhm}. These works show that neural operators are promising for electromagnetism problems, but they do not define a directly commensurate state-of-the-art for our combination of spacetime formulation, field generality, sparse observations, and exact constraints.

We are not aware of an off-the-shelf neural-operator baseline that is directly commensurate with our per-instance sparse-reconstruction setting. A from-scratch neural-operator comparison would require substantial upfront generation of accurate Maxwell data, followed by offline training before fast inference becomes possible~\cite{azizzadenesheli2024neuraloperators,iyer2026photonicdl}. Such a comparison would therefore probe a different computational regime---amortized surrogate learning over many problem instances---rather than the sparse-data, per-instance regime considered here.

\textbf{Gaussian Processes} (GPs) provide a natural Bayesian framework for linear PDEs, as linear operators preserve Gaussian structure. Early work incorporated exact linear constraints through latent parameterizations and specialized kernels, including latent force models, divergence-/curl-free constructions, LODE-GPs, and Maxwell-adjacent magnetic-field models \cite{alvarez2009latent,macedo2010learning,solin2018magnetic,jidling2017linearly,langehegermann2018algorithmic,besginow2022lode}. More recent work extended exact-prior ideas to constant-coefficient PDE systems and boundary conditions, including EPGP and boundary-constrained constructions \cite{harkonen2023gaussian,langehegermann2021boundary,langehegermann2022boundary,dalton2024boundary,huang2024bepgp,li2025gaussian}. As is the case for \methodname{}, these models embed the physical law  in the admissible function class rather than only inform the penalty terms. 
In contrast to these GPs, we realize exact Maxwell structure as an architectural inductive bias in a trainable shallow network, enabling direct gradient-based fitting from sparse data rather than expensive posterior inference over a GP prior.

\textbf{Hard-constrained Scientific ML and Exact-solution Spaces.}
Beyond residual-based training, several lines of work restrict approximation directly to physically admissible function classes \cite{lagaris1998annpde,schiassi2021xtfc,sukumar2022exactbc}. In wave computation, Trefftz and ultra-weak variational methods use local trial functions that satisfy the governing equations by construction \cite{trefftz1926gegenstueck,huttunen2007uwvf,hiptmair2016trefftzsurvey}. In modern ML, hard-constrained Maxwell architectures already exist, for example through potential-based or Fourier-space parameterizations \cite{scheinker2023pcnn,leon2024fhm}. Our method differs from both lines of work in that it is a shallow trainable architecture for per-instance sparse-data fitting, where each neuron is designed as a homogeneous Maxwell solution.

Related ideas therefore appear across residual-based scientific ML, hard-constrained Maxwell architectures, exact-prior GP models, and exact wave-based trial spaces; our contribution sits at their intersection as a rapidly trainable exact model class for sparse-data fitting.

\section{Experiments}
\label{experiments}

We present experiments designed to demonstrate the efficiency and accuracy of \methodname{} as a solver of Initial Value Problems  for \eqref{eq:maxwell} in comparison to prior art: PINNs, FEMs, and S-EPGP\footnote{Code for \methodname{} and all baselines discussed in this paper are available in the supplementary material.}. We perform this comparison in two settings: with and without boundary conditions (BC and IC, respectively), whenever possible. Note that FEMs can only operate in the BC setup. We then report and discuss the experimental results for four ground truth solutions to Maxwell's equations, which we term \textit{Plane Waves}, \textit{Radial Waves}, \textit{Hopf Fibration}, and \textit{Random Solution}. These solutions are chosen to be as different from our \eqref{eq:model} basis and from each other as possible, demonstrating the generalization capabilities of our network  (see Appendix~\ref{app:ground_truths} for details). 
We first describe our experimental setup and evaluation criteria and then our three experiments:
1) \textit{Which method can reach relative $L^2$-accuracy of 5\% fastest?}  2) \textit{What is the  accuracy each method achieves after 70s of training time?} 
3) \textit{What is the best accuracy \methodname{} achieves with only 100 data samples?}

We first pick a solution $\E_\gt,\B_\gt\colon[0,\infty)\times(0,1)^3\to\R^3$ from the list of four in Appendix~\ref{app:ground_truths}.

In the \textbf{initial value problems (IC) setup} we seek to reconstruct $\E,\B\colon [0,0.1]\times [0.2,0.8]^3\to\R^3$, solving \eqref{eq:maxwell} using 1,000 data points $(x,y,z)\in[0,1]^3$ at $t=0$:
\begin{align}\tag{IC}\label{eq:initial_conditions}
    \E(0,x,y,z)=\E_\gt(0,x,y,z)\quad\text{and}\quad \B(0,x,y,z)=\B_\gt(0,x,y,z).
\end{align}
The evaluation domain $[0,0.1]\times [0.2,0.8]^3$ is chosen to lie within the domain of dependence of the initial-data region $\{0\}\times [0,1]^3$, as prescribed by the finite speed of propagation of Maxwell's equations \cite[Theorem 3.12]{Lax2006}.
This local spacetime reconstruction setting is different from the standard fully specified forward problems for which FEMs are designed.

In the \textbf{initial-boundary value problems (BC) setup} we consider tangential boundary conditions
\begin{align}\tag{BC}\label{eq:boundary_condition}
    \E(t,x,y,z)\times \n = \E_\gt(t,x,y,z)\times \n\quad\text{for }t\in[0,1]\text{ and }(x,y,z)\in\partial[0,1]^3,
\end{align}
 where $\n$ is the outward normal to the boundary of the cube $\partial [0,1]^3$ (for instance, on the top of the cube, $\n = (0,0,1)$). We predict over a validation set sampled from $(t,x,y,z)\in[0,1]\times[0,1]^3$, which is consistent with the fact that the system \eqref{eq:maxwell} with initial conditions \eqref{eq:initial_conditions} and boundary conditions \eqref{eq:boundary_condition} is well-posed in $L^2((0,1)^4)$. In the BC setup we sample 1,000/7 points $(t,x,y,z)$ on each of the 7 faces of the cube $(0,1)^4$ which are considered in \eqref{eq:initial_conditions} and \eqref{eq:boundary_condition}.

\begin{table}[t!]
  \caption{\textbf{Experiment 1: Race to 5\%.} Validation Error, Residual Error, and Compute training time to reach $5\%$ relative validation error (if possible) for the ground truths: \textit{Plane Waves}, \textit{Radial Waves}, \textit{Hopf Fibration}, \textit{Random Solution} with/without boundary conditions (BCs). Animations to complement these results visually
  are available in the supplementary material at \texttt{flash\_max/flash\_max\_gifs}.
 }
  \label{main-results}
  \centering
  \begin{tabular}{crcrrr}
    \toprule
    & \textbf{Method}                      & \textbf{BC} & \textbf{Val Error (\%)} & \textbf{Residual Error} & \textbf{Compute (s)} \\
    \midrule
    \multirow{8}{*}{\rotatebox[origin=c]{90}{\textit{Plane Waves}}}
    & \multirow{2}{*}{FEM}                 & \checkmark  & $8.22$                  & $0.052$                 & $3{,}000$ \\
    &                                      & \xmark      & \texttt{n/a}            & \texttt{n/a}            & \texttt{n/a}  \\
    \cdashline{2-6}
    & \multirow{2}{*}{PINN}                & \checkmark  & $\bfnum{<5.00}$         & $0.007$                 & $304.2 \pm \phantom{0}73.2$ \\
    &                                      & \xmark      & $\bfnum{<5.00}$         & $0.013$                 & $119.3 \pm \phantom{0}28.5$ \\
    \cdashline{2-6}
    & \multirow{2}{*}{S-EPGP}              & \checkmark & $26.65$                  & $\bfnum{0}$             & $345.7 \pm \phantom{0}69.7$  \\
    &                                      & \xmark     & $25.04$                  & $\bfnum{0}$             & $61.4 \pm \phantom{0}25.5$ \\
    \cdashline{2-6}
    & \multirow{2}{*}{\methodname\ (ours)} & \checkmark & $\bfnum{<5.00}$          & $\bfnum{0}$             & $\bfnum{8.1} \pm \phantom{00}\bfnum{0.4}$ \\
    &                                      & \xmark     & $\bfnum{<5.00}$          & $\bfnum{0}$             & $\bfnum{6.8} \pm \phantom{00}\bfnum{0.3}$ \\
    \midrule
    \multirow{8}{*}{\rotatebox[origin=c]{90}{\textit{Radial Waves}}}
    & \multirow{2}{*}{FEM}                 & \checkmark & 65.47     & 2.002                      &  2.9 \\
    &                                      & \xmark     & \texttt{n/a}             & \texttt{n/a}            & \texttt{n/a} \\
    \cdashline{2-6}
    & \multirow{2}{*}{PINN}                & \checkmark & $93.9$                   & $0.063$                 &  $47.4 \pm \phantom{0}75.0$ \\
    &                                      & \xmark     & $86.7$          & $0.072$                 & $56.4 \pm \phantom{0}34.2$ \\
    \cdashline{2-6}
    & \multirow{2}{*}{S-EPGP}              & \checkmark & did not converge     & --                      &  -- \\
    &                                      & \xmark     & $5.85$                   & $\bfnum{0}$             &  $89.1 \pm \phantom{0}15.5$ \\
    \cdashline{2-6}
    & \multirow{2}{*}{\methodname\ (ours)} & \checkmark & $\bfnum{<5.00}$          & $\bfnum{0}$             & $\bfnum{58.9} \pm \phantom{00}\bfnum{2.0}$ \\
    &                                      & \xmark     &  $\bfnum{<5.00}$         & $\bfnum{0}$             & $\bfnum{11.9} \pm \phantom{00}\bfnum{0.6}$ \\
    \midrule
    \multirow{8}{*}{\rotatebox[origin=c]{90}{\textit{Hopf Fibration}}}
    
    & \multirow{2}{*}{FEM}                 & \checkmark & $\bfnum{<5.00}$          & $0.009$                 & $157$\\
    &                                      & \xmark     & \texttt{n/a}             & \texttt{n/a}            & \texttt{n/a} \\
    \cdashline{2-6}
    & \multirow{2}{*}{PINN}                & \checkmark &   $\bfnum{<5.00}$        & $0.010$                 & $235.3 \pm \phantom{0}80.7$ \\
    &                                      & \xmark     &  $\bfnum{<5.00}$         & $0.008$                 & $89.8 \pm \phantom{0}25.9$ \\
    \cdashline{2-6}
    & \multirow{2}{*}{S-EPGP}              & \checkmark & $37.16$                  & $\bfnum{0}$             & $410.4 \pm \phantom{0}78.9$ \\
    &                                      & \xmark     & $\bfnum{<5.00}$          & $\bfnum{0}$             & $153.3 \pm \phantom{0}44.8$ \\
    \cdashline{2-6}
    & \multirow{2}{*}{\methodname\ (ours)} & \checkmark &  $\bfnum{<5.00}$         & $\bfnum{0}$             & $\bfnum{7.4} \pm \phantom{00}\bfnum{0.3}$ \\
    &                                      & \xmark     &  $\bfnum{<5.00}$         & $\bfnum{0}$             & $\bfnum{3.2} \pm \phantom{00}\bfnum{0.0}$ \\
    \midrule
    \multirow{8}{*}{\rotatebox[origin=c]{90}{\textit{Random Soln.}}}
    
    & \multirow{2}{*}{FEM}                 & \checkmark & $\bfnum{<5.00}$                   & $0.022$                 & $13{,}701$ \\
    &                                      & \xmark     & \texttt{n/a}             & \texttt{n/a}            & \texttt{n/a} \\
    \cdashline{2-6}
    & \multirow{2}{*}{PINN}                & \checkmark &   $5.65$                 & $0.017$                 & $355.1\pm \phantom{0}31.1$\\
    &                                      & \xmark     &   $\bfnum{<5.00}$        & $0.011$                 & $131.9 \pm \phantom{0}19.4$ \\
    \cdashline{2-6}
    & \multirow{2}{*}{S-EPGP}              & \checkmark & $75.0$                   & $\bfnum{0}$             & $516.7\pm \phantom{00}7.7$ \\
    &                                      & \xmark     & $\bfnum{<5.00}$          & $\bfnum{0}$             & $115.9 \pm \phantom{00}2.1$ \\
    \cdashline{2-6}
    & \multirow{2}{*}{\methodname\ (ours)} & \checkmark &  $\bfnum{<5.00}$         & $\bfnum{0}$             & $\bfnum{58.3} \pm \phantom{00}\bfnum{3.2}$ \\
    &                                      & \xmark     &  $\bfnum{<5.00}$         & $\bfnum{0}$             & $\bfnum{6.1} \pm \phantom{00}\bfnum{0.2}$ \\
    \bottomrule
  \end{tabular}\label{tab:race_to_5percent}
\end{table}


\textbf{The evaluation metric} we consider is  relative $L^2$-error to the ground truth, defined as
\begin{align*}
   \texttt{error}_\gt= \text{RMSE}((\E,\B),(\E_\gt,\B_\gt))/\text{RMSE}((\E_\gt,\B_\gt),\mathbf{0}).
\end{align*}
We validate over sets sampled from $[0,0.1]\times[0.2,0.8]^3$ in the IC setup and from $[0,1]^4$ in the BC setup.
For the PDE residual (PINNs and FEMs), we also report the $L^2$-error, namely
\begin{align*}
    \texttt{error}_{\mathrm{PDE}}=\text{RMSE}((\partial_t \E-\nabla \times \B,
    \partial_t\B + \nabla\times \E,
    \nabla\cdot \B,
    \nabla \cdot \E),\mathbf{0}).
\end{align*}
In all experiments, we optimize each method's hyperparameters (including the number of training and/or collocation points) individually and report the best-performing setups. For example, while \methodname{} performed  well  with 1K data points, PINNs needed 10K, and S-EPGP needed 12K. Implementation details for all methods can be found in Appendices~\ref{app:detailed_flashmax_description},~\ref{app:implementation_details}. 



\begin{figure}[th!]
  \centering
  \includegraphics[width=\textwidth]{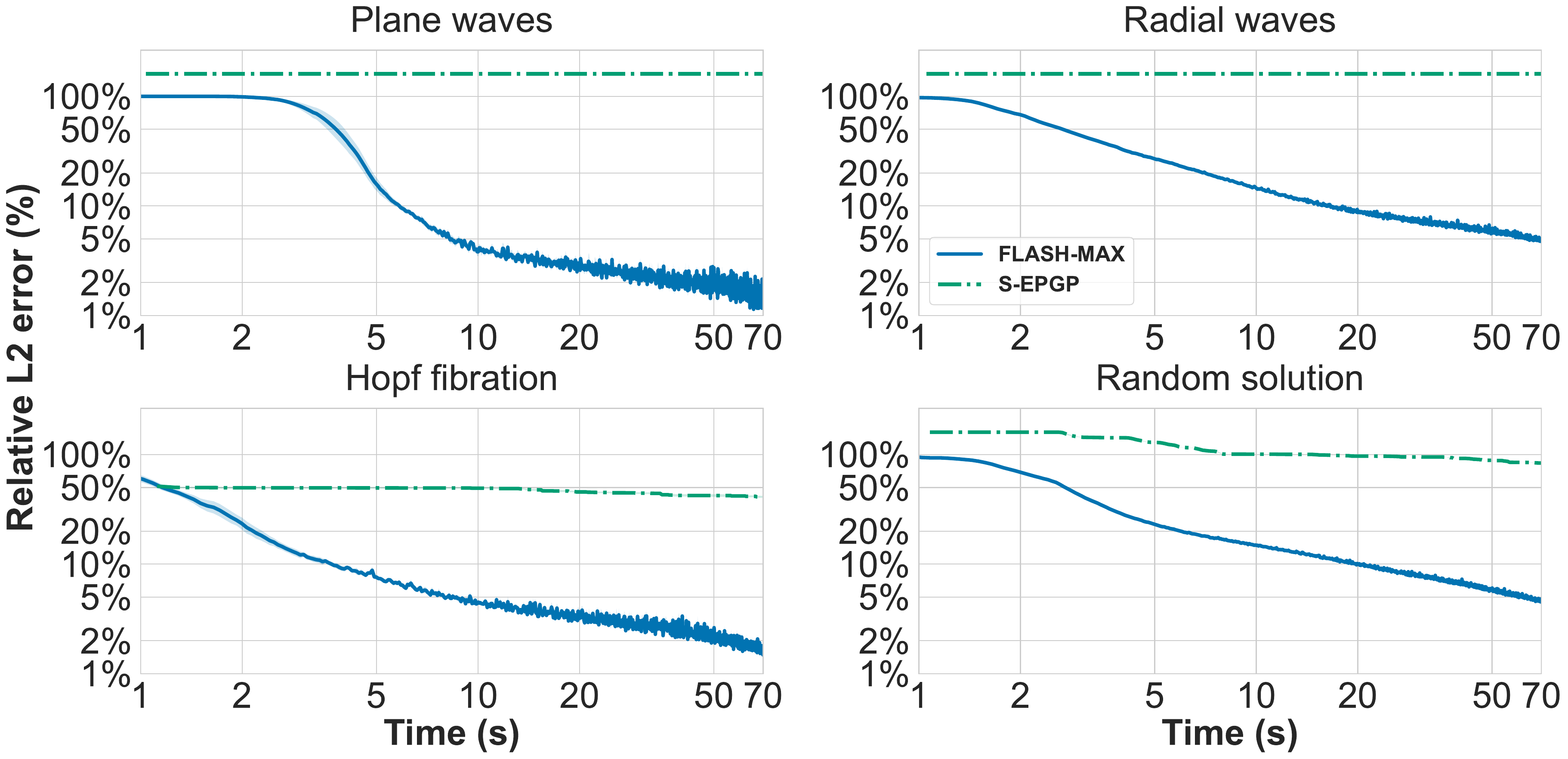}
  \caption{\textbf{Experiment 2: Performance on a Time Budget (BC setup).} The curves depict the achievable error rates ($y$-axis) given a specific time budget ($x$-axis) for the setting with boundary conditions. Uncertainties are shown (standard error of the mean over 5 random seeds). \methodname{} (ours, blue solid line) clearly converges significantly faster than S-EPGP  across all four problems.  FEM and PINN are excluded since they do not reach the PDE residual error threshold of $0.01$ during the time budget (cf. Table~\ref{tab:race_to_5percent}), which is a bare minimum to claim fidelity to the physics. 
  Plots for the IC setup can be found in the Appendix~\ref{app:additional_experiments}, Figure~\ref{fig:timing_ic} and follow a similar pattern. 
  }
  \label{fig:timing_bc}
\end{figure}


The \textbf{first experiment is a race to 5\%}. Given the target validation  $\texttt{error}_\gt<5\%$, we optimize each method, stopping as soon as the target validation error is reached (if it is reached at all). We further impose the normalized residual threshold \(\texttt{error}_{\mathrm{PDE}}<10^{-2}\), a generous threshold for
normalized Maxwell equations, reflecting the standard role of residuals as  indicators of physical fidelity in FEM analysis \cite{verfurth2013posteriori,Schoeberl2008} and neural operators \cite{mao2026accurate}. In contrast, \methodname{} and S-EPGP have residual error equal to 0 by construction. For each method, we report the validation error, residual error, and compute time   in both the BC and IC setups as applicable. 

Table~\ref{tab:race_to_5percent} presents our \emph{Race to 5\% results}, reporting validation error, residual error, and compute time (in seconds) across all methods, ground truths, and boundary condition setups. Experiments utilized different hardware: NVIDIA L40S GPU (\methodname{}, PINN), NVIDIA A100 GPU (S-EPGP), and M2 CPU (FEM), yet the comparison remains fair and underscores the algorithmic efficiency of \methodname{}. First, comparing \methodname{} and the in-house PINN on the exact same L40S GPU provides an apples-to-apples comparison against a primary machine learning baseline where our method achieves  orders-of-magnitude speedup. 
Second, S-EPGP remains significantly slower despite utilizing a more powerful NVIDIA A100 GPU to accommodate the intensive computational requirements of Gaussian process regression. 
Finally, FEM ran on a CPU, which is standard for FEniCS. To eliminate GPU acceleration as a confounder, we replicated \methodname{} on the same CPU hardware (Appendix~\ref{app:CPU}). \methodname{} still reached target accuracy substantially faster than FEM, proving that the speedup is architectural rather than hardware-driven.
In the sparse-data fast-training regime discussed here, \methodname{}  outperforms all competitor approaches \emph{by one to three orders of magnitude}. Notably, \methodname{} uses identical hyperparameter settings across all ground truths and both the IC and BC setups, achieving robust performance without the need to re-tune  hyperparameters on an instance-by-instance basis. PINNs and FEMs, unlike \methodname{}, do not yield identically zero residual error. PINNs are highly unstable during training, and although they sometimes converge to the 5\% validation error threshold, they frequently fail to do so, even with extensive hyperparameter tuning. S-EPGPs do not reach the 5\% validation error threshold in many of the analyzed setups. FEMs, which are not applicable to the IC setup, frequently require several orders of magnitude more training time to reach the 5\% validation error threshold, and sometimes do not converge to the threshold at all. 

The \textbf{second experiment is performance on a time budget} of 70s.
Figure~\ref{fig:timing_bc} depicts the validation error $\texttt{error}_\gt$ over time in a log-log scale. See also Figure~\ref{fig:timing_ic} in the appendix for more comparisons.
We see that our approach \methodname{} reduces the error faster and continues to consistently reduce the error during training time.
FEMs and PINNs do not reach the PDE residual error threshold of 0.01 and S-EPGP shows very slow convergence, as is to be expected for a Gaussian process method.

The \textbf{third experiment is performance on a data budget}.
We analyze the data requirements of \methodname{} in Table~\ref{tab:n_points_sensitivity}. We find that \methodname{} performs well even with very sparse observations, often reaching relative $L^2$-error below 1\% with just a few hundred training pairs. In all experiments shown in Table~\ref{tab:n_points_sensitivity}, we use the IC setup shown in Table~\ref{tab:race_to_5percent} (rows with \xmark), varying only the number of data points used in training and fixing the random seed to 42 across runs. We find  consistent, strong performance across all four ground truth solutions, with as few as 100 to 500 training observations.
%
We refer to Appendix~\ref{app:additional_experiments} for additional experiments, ablations, and sensitivity analyses.

\begin{table}[b]
    \centering
    \small
    \setlength{\tabcolsep}{3.5pt}
    \caption{\textbf{Experiment 3: Performance on a Data Budget}. Sensitivity of \methodname{} to the number of training points (\#pts) in the IC (initial conditions) setup.
    All runs use random seed 42. We write PW = \textit{Plane Waves}, RW = \textit{Radial Waves}, HF = \textit{Hopf Fibration},
    RS = \textit{Random Solution}.}
    \label{tab:n_points_sensitivity}
    \begin{tabular}{@{}rrrrr@{\qquad}rrrrr@{}}
    \toprule
    \multicolumn{5}{c}{\textbf{Min. RL2E (\%)}} &
    \multicolumn{5}{c}{\textbf{Time to min. RL2E (s)}} \\
    \cmidrule(lr){1-5}
    \cmidrule(lr){6-10}
    \textbf{\# pts.} & \textbf{PW} & \textbf{RW} & \textbf{HF} & \textbf{RS}
    & \textbf{\# pts.} & \textbf{PW} & \textbf{RW} & \textbf{HF} & \textbf{RS} \\
    \midrule
     100 & 2.6    & 7.4    & $<1.0$ & 2.0
         & 100 & 520.4 & 138.0 & 33.2 & 137.0 \\
     200 & $<1.0$ & 2.9    & 1.1    & $<1.0$
         & 200 & 61.4  & 144.6 & 14.2 & 67.7 \\
     500 & $<1.0$ & $<1.0$ & $<1.0$ & $<1.0$
         & 500 & 19.2  & 164.7 & 11.5 & 50.2 \\
     1,000 & $<1.0$ & $<1.0$ & $<1.0$ & $<1.0$
         & 1,000 & 39.6  & 105.8 & 12.5 & 53.3 \\
     2,000 & $<1.0$ & $<1.0$ & $<1.0$ & $<1.0$
         & 2,000 & 35.0  & 86.6  & 19.0 & 57.2 \\
     5,000 & $<1.0$ & $<1.0$ & 1.5    & $<1.0$
         & 5,000 & 33.2  & 94.5  & 11.3 & 70.5 \\
    10,000 & $<1.0$ & $<1.0$ & $<1.0$ & $<1.0$
         & 10,000 & 40.4 & 97.3  & 21.9 & 73.1 \\
    12,000 & 1.3    & $<1.0$ & $<1.0$ & $<1.0$
         & 12,000 & 37.7 & 109.6 & 27.4 & 71.3 \\
    \bottomrule
    \end{tabular}
\end{table}

\paragraph{Discussion.} The experiments identify several advantages of \methodname{} within the sparse-reconstruction setting. First, \methodname{} is one to three orders of magnitude faster than prior art on almost every method-problem experiment (an end-to-end neural-operator baseline including solver-generated training data and hyperparameter tuning could require weeks to months of computation, so we did not include it as a primary comparator). Second, \methodname{} has a solution guarantee (zero residual error), like S-EPGP. PINNs and FEMs on the other hand do not have solution guarantees. Third, \methodname{} only requires sparse data (hundreds to a thousand data points vs.\ more than 10 thousand for other methods). Fourth, \methodname{} is able to obtain consistent performance reaching the validation error threshold without specific hyperparameter tuning for each problem separately. The same hyperparameters generalize to all four problems (Appendix ~\ref{app:implementation_details}). 

Although our in-house PINN provides the strongest performance among the PINN-style baselines considered in this work, its behavior remains highly problem-dependent and can become unstable on more challenging solution families. In particular, for the \textit{Radial Waves} solution, the model exhibits very poor accuracy, with a relative error of \(93.9\%\) for the case with boundary conditions, together with a relatively large residual error of \(0.063\). Since well-trained runs in our setting typically attain residual errors around \(0.01\), this result indicates a clear optimization failure rather than merely a mismatch between residual minimization and solution accuracy. We also observe substantial variability in computational cost, e.g., \(47.4 \pm 75.0\), which further reflects instability across runs.

The S-EPGP baseline also shows clear limitations, especially in the boundary-constrained setting. This behavior suggests that the finite-feature S-EPGP approximation is not expressive enough to simultaneously satisfy the initial and boundary constraints over the full space-time domain. In practice, the augmented training system appears to introduce strong competition between these two sources of supervision, and the resulting approximation quality deteriorates markedly across all four solutions. In addition, in the case without boundary conditions, since performance degrades substantially when extrapolating to longer time horizons, evaluation for this baseline is restricted to the interior region \([0,0.01]\times[0.2,0.8]^3\). We note that this is a shorter temporal interval than that used for \methodname{}. These results indicate that, although S-EPGP provides a useful statistical baseline, its applicability to the boundary-constrained 4D Maxwell setting considered here is limited.

The FEM baseline exhibits a different limitation from the learning-based baselines. As a structure-preserving discretization, its behavior depends directly on mesh resolution, polynomial degree, and solution regularity. In the degree-2 setup reported in Table~\ref{tab:race_to_5percent}, FEM reaches the validation target for the \textit{Hopf Fibration} and \textit{Random Solution} cases, while \textit{Plane Waves} remain above the $5\%$ target. The \textit{Radial Waves} case is the most challenging FEM result, since further mesh refinement increased the validation error, consistent with a large near-origin feature of the spherical-pulse solution on $[0,1]^3$. 

\textbf{Limitations.}
Our work focuses on the homogeneous, source-free Maxwell system with constant material parameters. The exactness guarantee of \methodname{} relies on this structure: incorporating currents, charges, conductivity, heterogeneous/dispersive media, and interface conditions would require additional exact parametrizations or constraint mechanisms. Moreover, while \methodname{}'s predictions satisfy the governing PDEs by construction, the initial and boundary observations are  imposed through data fitting rather than symbolic enforcement. Thus, the accuracy of the recovered field also depends on  data samples,  finite network size, and  nonconvex optimization strategy.

\section{Conclusions}


In this work, we introduced \methodname, an exact-by-construction architecture for reconstructing homogeneous electromagnetic fields from sparse observations. By embedding Maxwell’s equations directly into the model class \eqref{eq:model}, all predictions satisfy the governing physics identically, enabling fast training and accurate reconstruction from limited data. Empirically, the method achieves sub-1\% error in seconds while guaranteeing a zero PDE residual, substantially improving optimization speed over residual-based approaches such as PINNs, see Tables~\ref{tab:race_to_5percent},~\ref{tab:n_points_sensitivity}. 

Theoretically, we showed in Theorem~\ref{thm:spacetimeUAT} that this exact solution class remains expressive via a  universality result in arbitrary domains, demonstrating that strict physical admissibility need not limit approximation power. Overall, these results indicate that moving governing structure from the loss into the hypothesis class can improve the trade-off between physical consistency and computational efficiency, suggesting a broader design principle for scientific machine learning.





\newpage
{
    \small
    \bibliographystyle{unsrtnat} 
    \bibliography{main}
}

\newpage
\appendix
\section{Additional experiments}\label{app:additional_experiments}
\subsection{Sensitivity Analysis and Ablations}\label{app:ablations}

We evaluate \methodname{} with different network widths and activation functions in Table~\ref{tab:width}.
The results are consistently good.

\begin{table}[h]
    \centering
    \caption{\methodname{}'s sensitivity to network width (Width) and activation function (Activation). In all these runs, we use the \textit{Random Solution} and a seed of 42. All other hyperparameters are fixed as in Table~\ref{tab:race_to_5percent}. Executed on a single NVIDIA L40S.}
    \begin{tabular}{rrrrrrr}
    \toprule
        \textbf{Width}  & \textbf{Activation} & \textbf{Time to <5\% RL2E.} & \textbf{Min. RL2E.} & \textbf{Time to min. RL2E.} \\
         \textbf{(thousands)} &  &  \textbf{(s)} &  \textbf{(\%)} & \textbf{(s)} \\
    \midrule
        0.1               &       tanh &                488.2 &           4.9 &                488.2 \\
        0.2               &       tanh &                297.0 &           3.2 &                501.2 \\
        0.5               &       tanh &                202.8 &           3.1 &                506.4 \\
        1.0               &       tanh &                126.3 &           2.6 &                495.1 \\
        2.0               &       tanh &                 98.8 &           2.1 &                475.2 \\
        5.0               &       tanh &                 78.9 &           2.0 &                492.2 \\
        10.0              &       tanh &                 53.9 &           1.6 &                499.3 \\
        12.0              &       tanh &                 50.4 &           1.6 &                474.7 \\
        10.0              &       relu &                326.1 &           3.9 &                503.4 \\
        10.0              &       silu &                118.2 &           2.2 &                496.7 \\
        10.0              &     cosine &                 13.5 &           1.1 &                227.9 \\
        10.0              &       gelu &                 69.5 &           2.2 &                538.1 \\
        10.0              &    sigmoid &                320.9 &           3.0 &                598.6 \\
    \bottomrule
    \end{tabular}
    \label{tab:width}
\end{table}

\subsection{Runtimes on CPU}\label{app:CPU}
We also note that \methodname{} is rather efficient even in the absence of a GPU. Table~\ref{tab:cpu} shows timing results for exactly the same setups as shown in Table~\ref{tab:race_to_5percent}, except that all computation is handled on a CPU. We note that errors well below the 5\% threshold are easily achieved in under 10 minutes of CPU training time in most cases, and in all other cases, the threshold is not far afield after that amount of training time. 
\begin{table}[h]
    \centering
    \caption{The same \methodname{} experiments as were carried out in Table~\ref{tab:race_to_5percent}, but executed on a CPU, an INTEL(R) XEON(R) GOLD 6542Y. We allowed at most 10 minutes of training (600 seconds) to each experiment.}
    \begin{tabular}{rrrr}
    \toprule
        \textbf{Solution} & \textbf{BC} & \textbf{Minimum Val Error (\%)} & \textbf{Compute (s)} \\
    \midrule
        \textit{Plane Waves}       & \checkmark  & $1.1 \pm 0.066$      & $551.2 \pm 14.3$ \\
        \textit{Plane Waves}       & \xmark      & $1.2 \pm 0.070$      & $477.1 \pm 46.8$ \\
        \textit{Radial Waves }     & \checkmark  & $5.4 \pm 0.081$      & $596.2 \pm 1.40$ \\
        \textit{Radial Waves}      & \xmark      & $1.9 \pm 0.069$      & $577.7 \pm 10.2$ \\
        \textit{Hopf Fibration}    & \checkmark  & $1.6 \pm 0.061$      & $582.8 \pm 10.3$ \\
        \textit{Hopf Fibration}    & \xmark      & $1.1 \pm 0.030$      & $165.5 \pm 7.90$ \\
        \textit{Random Solution}   & \checkmark  & $5.7 \pm 0.015$      & $596.4 \pm 2.20$ \\
        \textit{Random Solution}   & \xmark      & $1.2 \pm 0.044$      & $446.6 \pm 78.4$ \\
    \bottomrule
    \end{tabular}
    \label{tab:cpu}
\end{table}

\subsection{{Experiment 2}: Performance on a Time Budget (IC setup)}

The main paper includes performance over time for BC in Figure~\ref{fig:timing_bc}.
We also include the same results for the experiments using only initial conditions (IC) here in Figure~\ref{fig:timing_ic}.

\begin{figure}[th!]
  \centering
  \includegraphics[width=\textwidth]{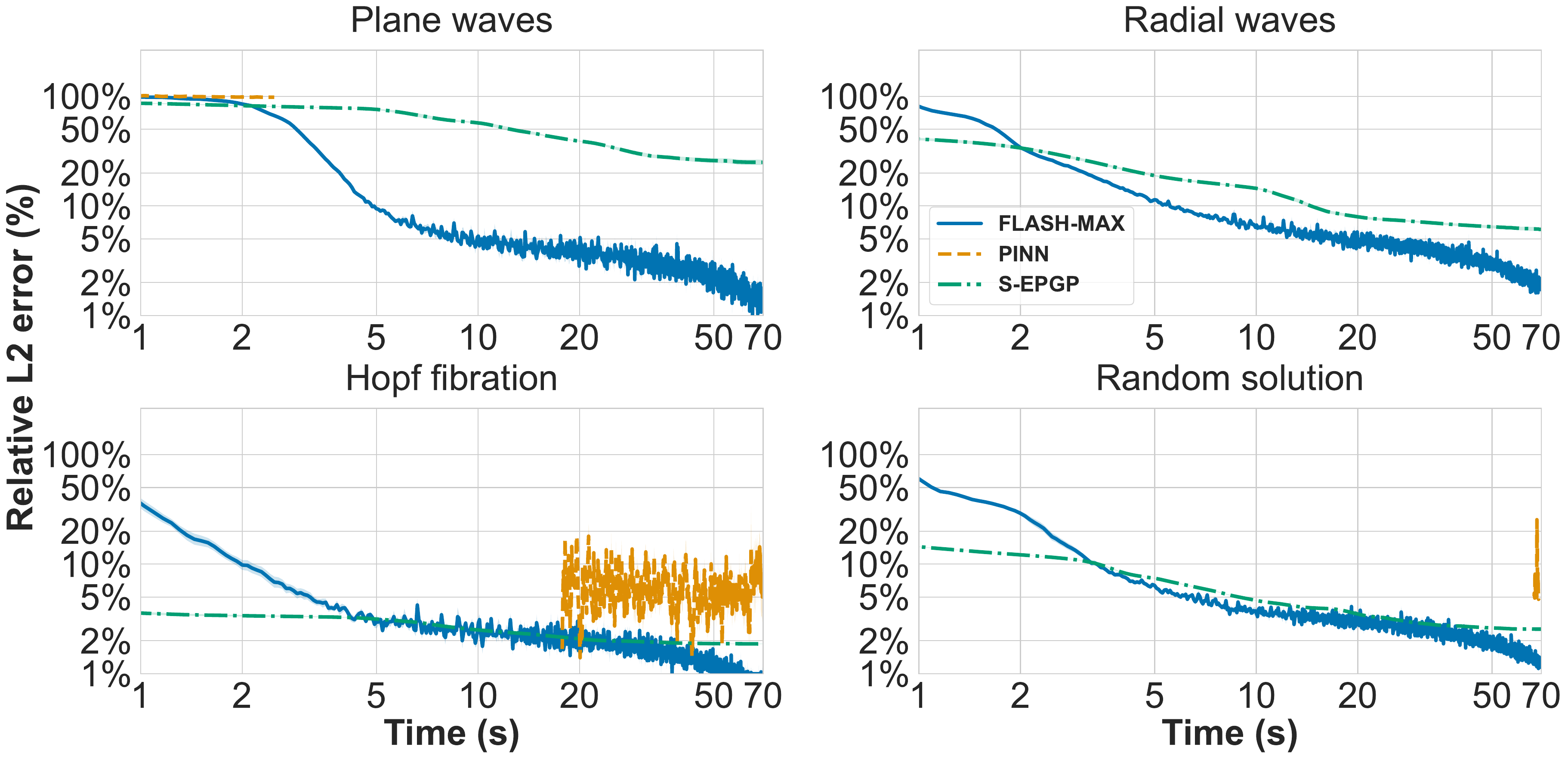}
  \caption{\textbf{Experiment 2: Performance on a Time Budget (IC setup).} The curves depict the achievable error rates ($y$-axis) given a specific time budget ($x$-axis) for the setting with initial conditions. Uncertainties are shown (standard error of the mean over 5 random seeds). \methodname{} (blue solid line) converges more reliably and usually significantly faster than existing methods across all four problems: \textit{Plane waves}, \textit{Radial waves}, \textit{Hopf fibration}, and \textit{Random solution}. Plots are analogous to those for the BC setup in Figure~\ref{fig:timing_bc}. While S-EPGP seems to perform comparably here, we note that it required an NVIDIA A100 to achieve this performance, while \methodname{} is achieving this performance on an NVIDIA L40S.}
  \label{fig:timing_ic}
\end{figure}

\section{Implementation Details}
\subsection{Detailed \methodname{} Description}\label{app:detailed_flashmax_description}

We also express \methodname{} in matrix notation: Let $\Z_i\in\R^{2W\times4}$ for $i\in\{1,2\}$ with entries $\Z_{ik\ell}$  having  $\ell$th column denoted by $\Z_{i\bullet\ell}\in\R^{2W}$ for $\ell\in\{0,1,2,3\}$ and vectors $\boldsymbol{b}_i,\w_i\in\R^{2W}$ with entries $\boldsymbol{b}_{ik}$ resp.\ $\w_{ik}$. 
For notational simplicity, we compress the notation in \eqref{eq:model} by absorbing the $j$ index in \eqref{eq:char_var} directly in the definition of our special frequencies $$\Z_{ik0}=\pm\sqrt{\Z_{ik1}^2+\Z_{ik2}^2+\Z_{ik3}^2}$$
with half the signs being $+$ resp. $-$. This simplification requires us to double the width parameter $W$ from \eqref{eq:model} (instead of working with indices $(j,k)$ with $j=\pm1$, $k=1,\ldots W$, we now only work with indices $k=1,\ldots 2W$).
Furthermore, we express~\eqref{eq:noetherian_multipliers} using the $2W\times 6$ matrices
\begin{align*}
    \p_1\left(\Z_1\right) = \begin{bmatrix}
        -\Z_{1\bullet1} \odot \Z_{1\bullet3} \\
        -\Z_{1\bullet2} \odot \Z_{1\bullet3} \\
        \Z_{1\bullet0} \odot \Z_{1\bullet0} - \Z_{1\bullet3} \odot \Z_{1\bullet3} \\
        -\Z_{1\bullet0} \odot \Z_{1\bullet2} \\
        \Z_{1\bullet0} \odot \Z_{1\bullet1} \\
        \mathbf{0} \\
    \end{bmatrix},\,
    \p_2\left(\Z_2\right) = \begin{bmatrix}
        \Z_{2\bullet1} \odot \Z_{2\bullet2} \\
        -\Z_{2\bullet0} \odot \Z_{2\bullet0} + \Z_{2\bullet2} \odot \Z_{2\bullet2} \\
        \Z_{2\bullet2} \odot \Z_{2\bullet3} \\
        -\Z_{2\bullet0} \odot \Z_{2\bullet3} \\
        \mathbf{0} \\
        \Z_{2\bullet0} \odot \Z_{2\bullet1} \\
    \end{bmatrix} 
\end{align*}
and entry-wise application of the activation $\sigma$ we obtain~\eqref{eq:model} at $\x\in\R^4$
\begin{align}\label{eq:forward_pass}
    \begin{pmatrix}\E_\theta(\x)\\\B_\theta(\x)\end{pmatrix}=\w_1^{\top}\left(\sigma(\x^{\top}\Z_1+\boldsymbol{b}_1)\odot \boldsymbol{p}_1(\Z_1)\right)
    +
    \w_2^{\top}\left(\sigma(\x^{\top}\Z_2+\boldsymbol{b}_2)\odot \boldsymbol{p}_2(\Z_2)\right)\in\R^6,
\end{align}
where $\theta=\{\boldsymbol{Z}_{ik1},\boldsymbol{Z}_{ik2},\boldsymbol{Z}_{ik3},\boldsymbol{b}_{ik},\boldsymbol{w}_{ik}\colon i=1,2,\,k=1,\ldots 2W\}$ are ordinary trainable weights.

The forward pass of \methodname{} described above is also represented in Figure~\ref{fig:forward_pass}.

We view~\eqref{eq:model} and its implementation as a shallow neural network with two bells and whistles attached: the weight sharing in the $\z_{ik0}$ being determined by other entries and the multiplication by the $\boldsymbol{p}(\z_i)$.
These two bells and whistles ensure the properties regarding Maxwell's equations~\eqref{eq:maxwell}
Most importantly, the model in~\eqref{eq:model} only yields solutions of~\eqref{eq:maxwell}, which is proved by direct calculation in Lemma~\ref{lem:solution}).
Moreover, \methodname{} has very few hyperparameters, as we do not have additional layers, nor a need for collocation points to enforce the equations.

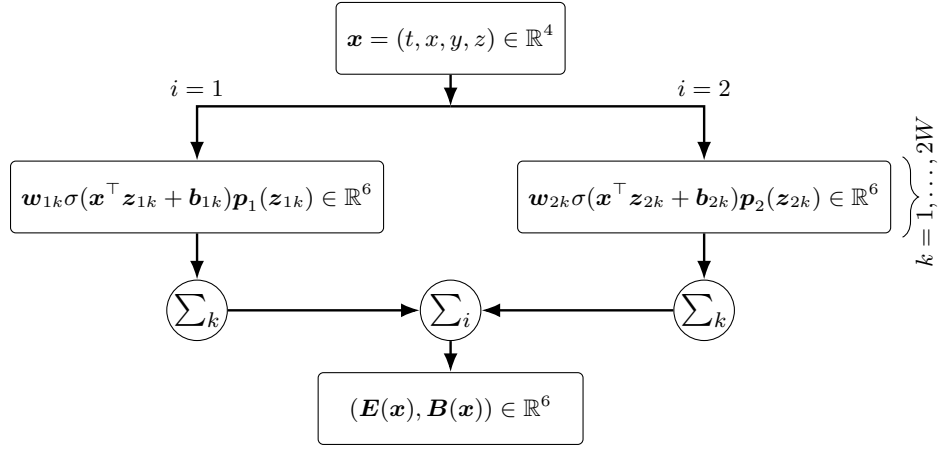
\begin{figure}
\centering
\begin{tikzpicture}[font=\small]

\node[io] (input) at (0,-1.4) {$\x=(t,x,y,z)\in\R^4$};
\coordinate (split) at (0,-2.3);

\node[pbox] (pol1) at (-3.35,-3.5) {$\boldsymbol{w}_{1k}\sigma(\x^{\top}\z_{1k}+\boldsymbol b_{1k})  \p_1(\z_{1k})\in\R^6$};
\node[pbox] (pol2) at (3.35,-3.5) {$\boldsymbol{w}_{2k}\sigma(\x^{\top}\z_{2k}+\boldsymbol b_{2k})  \p_2(\z_{2k})\in\R^6$};

\node[sum] (sum1) at (-3.35,-5.0) {$\sum_k$};
\node[sum] (sum2) at (3.35,-5.0) {$\sum_k$};
\node[sum] (sumi) at (0,-5.0) {$\sum_i$};
\node[io, minimum width=3.5cm] (out) at (0,-6.3) {$(\E(\x),\B(\x))\in\R^6$};

\draw[conn] (input.south) -- (split);
\draw[conn] (split) -| node[above]{$i=1$} (pol1.north);
\draw[conn] (split) -| node[above]{$i=2$}(pol2.north);
\draw[conn] (pol1) -- (sum1);
\draw[conn] (pol2) -- (sum2);
\draw[conn] (sum1) -- (sumi);
\draw[conn] (sum2) -- (sumi);
\draw[conn] (sumi) -- (out);

\coordinate (k2top) at ($(pol2.north east)+(0.15,0.05)$);
\coordinate (k2bot) at ($(pol2.south east)+(0.15,-0.05)$);

\draw[
  decorate,
  decoration={brace,amplitude=6pt,mirror}
] (k2bot) -- node[right=9pt, yshift=-30pt, rotate=90] {$k=1,\dots,2W$} (k2top);
\end{tikzpicture}
\caption{Graphical representation of the forward pass \eqref{eq:forward_pass} of \methodname{}, here expressed much like in \eqref{eq:model} in terms of each timespace frequency $\z_{ik}=\Z_{ik\bullet}\in\R^4$ for $i=1,2$, $k=1,\ldots,2W$.}\label{fig:forward_pass}
\end{figure}






\paragraph{\methodname{} training setup.}
In all \methodname~experiments in Table~\ref{tab:race_to_5percent}, we use the AdamW optimizer with a learning rate of $5\times 10^{-2}$, a weight decay of $5\times 10^{-5}$ and $(\beta_1,\beta_2)=(0.9,0.95)$. The activation function is $\sigma=\tanh$. We initialize all weights with the Xavier normal initialization with a standard gain of $5/3$, except for the biases which are initialized with $0$. We use a model width of $w=$ 10K (corresponding to 200K trainable parameters) and train with randomized batches of size 1K over 2K training points. Validation error calculations are over 10K points sampled from the validation domain. A cosine learning rate scheduler was employed over the first 10K epochs of training with $\eta_{\min{}} = 0$.

In the boundary condition setup, $1/7$ of the training points are sampled from the interior of the training spatial domain $(x,y,z) \in [0,1]^3$ with $t=0$ (initial condition points - IC points), while $1/7$ of the training points are sampled from each of the six faces of the training spatial domain $(x,y,z) \in \partial [0,1]^3$ with $t\in[0,1]$ (boundary condition points - BC points). Validation data are drawn from $(t,x,y,z)\in[0,1]^4$. We mask the targets during training such that loss for BC points is computed only for the non-normal components of the electric field, effectively minimizing only $\text{MSE}\left(\E\times\n, \E_{\gt}\times\n\right)$ for BC points. For instance, considering a BC point from the top face of $[0,1]^3$, we compute the loss during training only over the components $\E_x$ and $\E_y$ as the normal direction $\E_z$ is not part of the boundary condition. We never include the $\B$ field for BC points--it is only used for the IC points.

In the IC setup, all training points are sampled from the interior of the training spatial domain $(x,y,z) \in [0,1]^3$ with $t=0$. Validation data are drawn from $(x,y,z)\in[0.2,0.8]^3$ with $t\in[0,0.1]$. No special masking is done.
\subsection{State-of-the-art Alternative Approaches}
\label{app:implementation_details}
The three main approaches we compare with are PINNs, Gaussian Processes (S-EPGP), and FEMs.

\paragraph{PINNs.}

We benchmark \methodname{} against a Physics-Informed Neural Network (PINN)~\cite{raissi2019physics} on the time-dependent Maxwell system in Eq.~\ref{eq:maxwell}. 
Unless otherwise specified, all PINN experiments use the same architecture and optimization setup: a 6-layer MLP with hidden size 128, Adam optimizer with learning rate \(10^{-2}\), and joint optimization of all active loss terms. For physics supervision, we construct a uniform Cartesian grid over \([0,1]^4\) and randomly sample 50K collocation points. For data supervision, initial-condition points are sampled from the \(t=0\) slice using a \(50\times 50\times 50\) Cartesian grid over \((x,y,z)\in[0,1]^3\), from which 10K training points are randomly selected.

In this work, we consider two settings that align with the problem formulations. In the~\ref{eq:initial_conditions} setting, the model is trained with initial data, PDE, and divergence losses, with weights \(\lambda_{\text{Data}}=10\), \(\lambda_{\text{PDE}}=1\), and \(\lambda_{\text{div}}=2\), and no boundary supervision is imposed. 
In the~\ref{eq:boundary_condition}  setting, we additionally enforce the boundary loss with weight \(\lambda_{\text{BC}}=10\), while keeping the other loss weights unchanged. Here, PDE loss refers to the residuals of the first two equations in the~\ref{eq:maxwell}; and \(\lambda_{\text{div}}\) is the weight for divergence loss corresponding to the residuals for the last two equations in the~\ref{eq:maxwell}. Boundary samples are drawn from the spatial boundary of the unit cube over \(t\in[0,1]\), and validation includes 10K sampled boundary points together with evaluation over the full space-time domain \([0,1]^4\).

\paragraph{S-EPGP.}

We benchmark \methodname{} against the S-EPGP baseline~\cite{harkonen2023gaussian}, following the same formulation and Maxwell-consistent feature construction as in~\cite{harkonen2023gaussian,li2025gaussian}. In all experiments, S-EPGP is trained by minimizing the negative log marginal likelihood (NLML), and relative error is computed jointly over all six electromagnetic components. Unless otherwise specified, all S-EPGP experiments use the same training setup: Adam optimization on an A100 GPU, \(n_{\mathrm{MC}}=300\) frequency samples, and \(8n_{\mathrm{MC}}=2400\) spectral features. Initial-condition supervision is constructed from the same \(t=0\) Cartesian grid as above, from which we sample 2K locations, corresponding to 12K scalar observations over the six field components.

We consider two settings aligned with the formulations in this work. In the~\ref{eq:initial_conditions} setting, the model is trained using only observations at \(t=0\).
Since performance degrades substantially when extrapolating to longer time horizons, evaluation for this baseline is restricted to the interior region \([0,0.01]\times[0.2,0.8]^3\). We note that this is a shorter temporal interval than that used for \methodname{}, reflecting the limited stability of the finite-feature S-EPGP approximation. In the~\ref{eq:boundary_condition} setting, we additionally impose boundary constraints using 1K sampled points per spatial face. The tangential electric-field constraints are incorporated by augmenting the S-EPGP feature matrix and observation vector, and evaluation is performed over the full space-time domain \([0,1]^4\).

The relatively weak performance of S-EPGP, especially in the~\ref{eq:boundary_condition} setting, likely arises from the finite feature approximation. While the Maxwell-consistent basis is theoretically expressive in the infinite-feature limit, \(n_{\mathrm{MC}}=300\) appears insufficient to accurately represent the target solution over the full constrained domain. As a result, when initial and boundary constraints are combined in the augmented system, basis mismatch can lead to tension between the two objectives, whereas fitting the initial slice alone is substantially easier. 

\paragraph{FEM.}

We benchmark \methodname{} against a classical structure-preserving finite element baseline~\cite{nedelec1980mixed,arnold2006finite} implemented in DOLFINx/FEniCS~\cite{logg2012fenics}. The FEM baseline is used only in the boundary condition setting, since the no boundary condition setting considered in this paper is a sparse-data extrapolation task rather than a fully specified forward initial-boundary value problem. Consistent with the structure-preserving FEM perspective discussed in Section~2, the solver uses a mixed $H(\mathrm{curl}) \times H(\mathrm{div})$ discretization of the source-free Maxwell system~\eqref{eq:maxwell}, where the electric field $\E$ is represented with first-kind N\'ed\'elec elements and the magnetic field $\B$ with Raviart--Thomas elements.

Time integration is performed with a Crank--Nicolson scheme. Since the benchmark problem is linear and the mesh is fixed, the system matrix is assembled once and reused throughout the time-stepping loop. All FEM runs use a structured tetrahedral mesh of $[0,1]^3$, with $N$ subdivisions per coordinate direction and $6N^3$ tetrahedra. We use matching polynomial degree-2 spaces for $\E$ and $\B$, time step $\Delta t = 10^{-2}$, and a PETSc GMRES solver with block-Jacobi ILU preconditioning.

Initial data are imposed by interpolating the exact electric and magnetic fields at $t=0$ into the corresponding finite element spaces. The tangential electric boundary condition $\E \times \n$ is imposed strongly on $\partial[0,1]^3$ at every time step using
the exact trace; no boundary condition is imposed directly on $\B$. Validation is performed on 10K randomly sampled spacetime points in $[0,1]^4$, using the same relative error metric as the other methods and computing the error jointly over all six electromagnetic components. We additionally report a final-step FEM residual computed from the last two Crank--Nicolson time levels.

The results of our FEM implementation on the four ground truths (Appendix~\ref{app:ground_truths}) are reported in Table~\ref{tab:race_to_5percent}. Note that the performance of FEM on the \textit{Radial Waves} solution was very poor. This is due to a numerical artifact at $(0,0,0,0)$, see Figure~\ref{fig:radial_waves} (top-left). The ML methods we use do not seem to be affected by the artifact due to the random sampling of data. To demonstrate that our FEM converges, albeit slowly, we approximate the \textit{Radial Waves} ground truth by shifting the domain away from the singularity, taking $\Omega=(0.1,1.1)^3$. FEM reaches sub-5\% validation error, but stabilizes well above the PDE residual threshold of $0.01$ reported (Table~\ref{tab:radial-waves-shifted-fem}), thus remaining far from the performance of \methodname{} for \textit{Radial Wave} with boundary conditions in Table~\ref{tab:race_to_5percent}.
\begin{table}[th!]
  \caption{Additional experiment for FEM. We repeat the \textit{Race to 5\%} experiment in Table~\ref{tab:race_to_5percent} on the \textit{Radial Waves} ground truth. Spacetime domain $(0,1)\times(0.1,1.1)^3$, designed to avoid the numerical artifact at the origin in Figure~\ref{fig:radial_waves}(top-left).}
  \label{tab:radial-waves-shifted-fem}
  \centering
  \begin{tabular}{crrrrr}
    \toprule
    \textbf{Ground Truth} & \textbf{Method} & \textbf{Space} & \textbf{Val Error (\%)} & \textbf{Residual Error} & \textbf{Compute (s)} \\
    \midrule
    \textit{Radial Waves} & FEM & $(0.1,1.1)^3$ & $\bfnum{<5.00}$ & $0.140$ & $287.4$ \\
    \bottomrule
  \end{tabular}
\end{table}

\subsection{Physics-Informed Neural Network (PINN) Baseline: QPINN and Evo-PINN}
\label{app:pinn_baseline}
\begin{table}[b]
  \caption{Validation Error, Residual Error, and Compute training time to reach the best relative validation error.}
  \label{Tab: QPINNandEvo}
  \centering
  \begin{tabular}{crccccc}
    \toprule
    \multirow{6}{*}{\rotatebox[origin=c]{90}{\textit{PlaneWave}\hspace*{1em}}}
    & \textbf{Method}                       & \textbf{BC}       & \textbf{Validation Error} & \textbf{Residual Error}       & \textbf{Compute} & \textbf{Device}        \\
    \midrule
    & \multirow{2}{*}{QPINN\cite{chen2026quantum}}        & \checkmark &  $48.07\%$        & $ 9.16\times 10^{-2}$ & $348.55 $s & \multirow{2}{*}{A100} \\
    &                              & \xmark     &  $96.06\%$        & $ 8.38\times 10^{-2}$ &  $ 37.26 $s &  \\
    \cdashline{2-7}
    & \multirow{2}{*}{Evo-PINN\cite{11353124}}         & \checkmark & $100.17\%$        & $ 7.7\times 10^{-3}$ &  35.87s & \multirow{2}{*}{CPU} \\
    &                              & \xmark     & $99.10\%$     & $1.6 \times 10^{-1}$         & $146.29s$ &  \\
   
    \midrule
    \multirow{6}{*}{\rotatebox[origin=c]{90}{\textit{RadialWave}\hspace*{-2.2em}}}
    & \multirow{2}{*}{QPINN\cite{chen2026quantum}}        & \checkmark &  $99.68\%$        & $6.46\times 10^{-4}$ &   $ 7.14 $s & \multirow{2}{*}{A100}\\
    &                              & \xmark     &   $52.37\%$        & $ 2.47\times 10^{-1}$ &  $ 36.23 $s \\
    \cdashline{2-7}
    & \multirow{2}{*}{Evo-PINN\cite{11353124}}         & \checkmark & $101.81\%$        & $ 5.1\times 10^{-3}$ &  $35.81$s & \multirow{2}{*}{CPU}\\
    &                              & \xmark     & $96.30\%$     & $7.3\times 10^{-2}$         & $87.77s$ \\

    \midrule
    \multirow{6}{*}{\rotatebox[origin=c]{90}{\textit{Hopf}\hspace*{-2.5em}}}
    & \multirow{2}{*}{QPINN\cite{chen2026quantum}}        & \checkmark &   $36.92\%$        & $ 6.42\times 10^{-2}$ &  $ 379.71 $s &\multirow{2}{*}{A100} \\
    &                              & \xmark     &  $8.71\%$        & $4.84 \times 10^{-2}$ &  $ 1186.49 $s \\
    \cdashline{2-7}
    & \multirow{2}{*}{Evo-PINN\cite{11353124}}         & \checkmark & $83.86\%$        & $ 1.0\times 10^{-3}$ &  18.85s & \multirow{2}{*}{CPU}\\
    &                              & \xmark     & $46.88\%$     & $1.9\times 10^{-2}$        & 73.1s \\

    \midrule
    \multirow{6}{*}{\rotatebox[origin=c]{90}{\textit{Random}\hspace*{-2em}}}
    & \multirow{2}{*}{QPINN\cite{chen2026quantum}}        & \checkmark &   $23.60\%$        & $ 3.83\times 10^{-2}$ & $956.27 $s & \multirow{2}{*}{A100}\\
    &                              & \xmark     &   $37.30\%$        & $ 4.46\times 10^{-2}$ &  $ 913.69 $s \\
    \cdashline{2-7}
    & \multirow{2}{*}{Evo-PINN\cite{11353124}}         & \checkmark & $103.06\%$        & $ 9.9\times 10^{-4}$ & 20.95s & \multirow{2}{*}{CPU}\\
    &                              & \xmark     & $73.90\%$     & $2.9\times 10^{-1}$         &     $ 2549s$ \\

    \bottomrule
  \end{tabular}
\end{table}

As discussed in Section~\ref{RelatedWork}, we further compare our in-house PINN baseline with two representative PINN-style methods, QPINN~\cite{chen2026quantum} and Evo-PINN~\cite{11353124}. Among these PINN-based baselines, our in-house PINN consistently achieves the strongest performance and is therefore used as the primary PINN comparator in the main text. For completeness, we report the detailed results of QPINN and Evo-PINN in Table~\ref{Tab: QPINNandEvo}.

To the best of our knowledge, no prior PINN work with publicly available code directly targets the general homogeneous, constant-coefficient 4D time-dependent Maxwell problem in Equation~\eqref{eq:maxwell}. We therefore include QPINN and Evo-PINN as the closest reproducible alternatives we could identify. QPINN is the most closely related task-specific baseline, but it is developed for a 2D hybrid quantum-classical Maxwell formulation with several problem-specific design choices. Evo-PINN, in contrast, is a recent general framework aimed at improving PINN optimization rather than a method specialized for Maxwell equations in our formulation. Although these baselines are not fully matched to our problem, they provide useful reference points for evaluating how related or more general PINN methodologies transfer to our setting. In our in-house PINN experiments, we do not add any additional techniques, but follow the tight alignment with the exact~\ref{eq:maxwell} formulation and training regime.

We adapt QPINN by retaining its core idea of inserting a parameterized quantum circuit near the output of an otherwise classical PINN, while replacing the original 2D three-field formulation \((E_z,H_x,H_y)\) with our six-component representation \((B_1,B_2,B_3,E_1,E_2,E_3)\). Our implementation uses a 6-layer MLP with hidden width 128, whose penultimate features are projected to a 7-qubit, 4-layer quantum circuit with Pauli-\(Z\) readout, using the strongly entangling ansatz and \texttt{acos} angle scaling. We also retain the additional global energy regularization used in the original QPINN study. Since several components of the original 2D setup are not directly transferable to our problem, this baseline should be viewed as an approximate transfer of the QPINN design rather than a matched reproduction.

We adapt Evo-PINN by replacing standard gradient-based PINN training with a CMA-ES optimizer implemented through EvoJAX, while keeping the PINN itself as a classical network over the same Maxwell variables and residual terms used in our benchmark. The resulting model is a smaller MLP trained by evolutionary search under a physics-informed objective composed of data, PDE, divergence, and, when applicable, boundary losses. This baseline follows the optimization-oriented spirit of Evo-PINN, but should be interpreted as an Evo-PINN-inspired evolutionary PINN rather than a like-for-like reproduction of a method designed for our Maxwell formulation. Moreover, under our computational budget, the evolutionary baseline uses both a smaller model and a more limited training budget than our in-house PINN, which likely further constrains its final performance.

\section{List of Ground Truths}\label{app:ground_truths}
We test each solver (ours, PINN, GP, FEM) using the  four ground truths that we describe in detail below. These vector fields are chosen to challenge every solver and, in particular, none of them  favor \methodname{} a priori. Moreover, these ground truths are very different mathematically from each other. \methodname{} substantially outperformed competitors for any choice of  ground truth and across all experiments we performed.
\begin{figure}[h]
    \centering
    \includegraphics[width=0.7\textwidth]{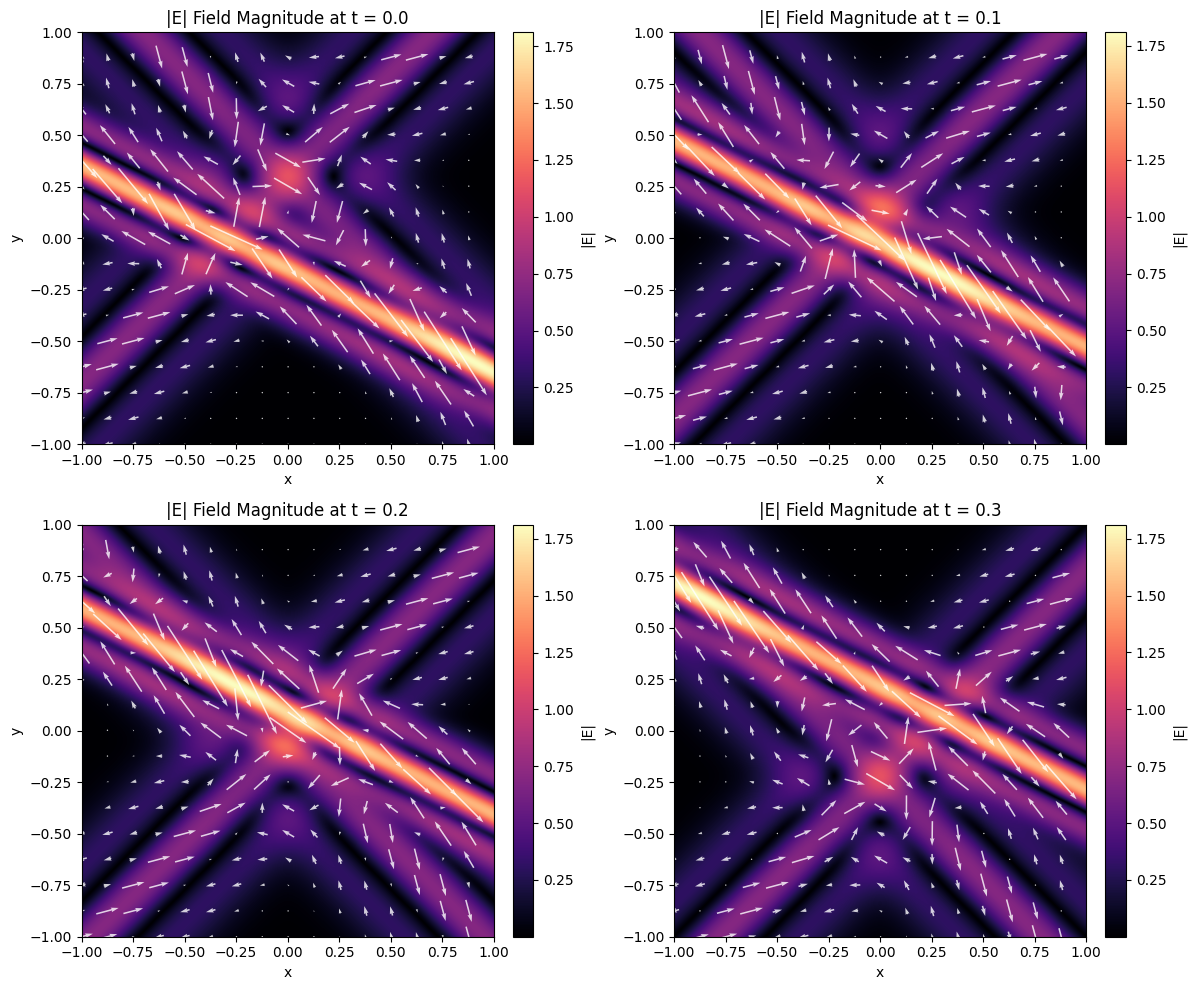}
    
    \caption{Time evolution of the magnitude of the electric field for the \textit{Plane Waves} solution.}
    \label{fig:plane_waves} 
\end{figure}
\paragraph{Plane Waves.} The solution of \eqref{eq:maxwell} with simplest structure that we use is a superposition of three plane wave solutions with profile derived from a  squared exponential
 $f(s) = 0.01 e^{-10(s - 0.3)^2}$. Let
\begin{equation*}
    F(s)=f''(s) = -0.2 e^{-10(s - 0.3)^2} \left[ 1 - 20(s - 0.3)^2 \right]
\end{equation*}
The directions of the three plane waves $\z$ must satisfy the light cone restriction $z_0^2=z_1^2+z_2^2+z_3^2$. We consider $\z\in\{(\sqrt{3},1,1,1),(\sqrt{3},-1,1,1),(\sqrt{6},-1,-2,1)\}$. Let
\begin{align*}
    s_1 &= \sqrt{3}t + x + y + z \\
    s_2 &= \sqrt{3}t - x + y + z \\
    s_3 &= \sqrt{6}t - x - 2y + z
\end{align*}

Let us define the evaluated functions as $F_1 = F(s_1)$, $F_2 = F(s_2)$, and $F_3 = F(s_3)$.
Recall that $\x=(t,x,y,z)$. The \textbf{electric field} $\E$ is given by
\begin{align*}
    E_1(\x) &=  (1 - \sqrt{3})F_1 - (1 + \sqrt{3})F_2 + (1 - 2\sqrt{6})F_3  \\
    E_2(\x) &=  (1 + \sqrt{3})F_1 + (1 - \sqrt{3})F_2 + (2 + \sqrt{6})F_3  \\
    E_3(\x) &=   -2F_1 - 2F_2 + 5F_3,
\end{align*}
while the \textbf{magnetic field} $\B$ is given by
\begin{align*}
    B_1(\x) &=  (\sqrt{3} + 1)F_1 + (\sqrt{3} - 1)F_2 + (2\sqrt{6} + 1)F_3 \\
    B_2(\x) &=  (1 - \sqrt{3})F_1 + (1 + \sqrt{3})F_2 + (2 - \sqrt{6})F_3  \\
    B_3(\x) &=   -2F_1 - 2F_2 + 5F_3.
\end{align*}
Snapshots of the evolution of the magnitude of the electric field can be found in Figure~\ref{fig:plane_waves}.

\paragraph{Radial Waves.} We next consider the evolution of a symmetric radial wave, modelling a Gaussian pulse at $(0,0,0)$ in space leading to waves which dissipate outward. This solution is fundamentally different than our network \eqref{eq:model}.

Let $r = \sqrt{x^2 + y^2 + z^2}$ and $s = r - t$.
The radial profile is given by
\begin{align*}
    f(s) &= 0.01 e^{-10(s - 0.7)^2} \\
    f'(s) &= -0.2(s - 0.7) e^{-10(s - 0.7)^2} \\
    f''(s) &= -0.2 e^{-10(s - 0.7)^2} \left[ 1 - 20(s - 0.7)^2 \right]
\end{align*}
We will use the following auxiliary scalar fields
\begin{align*}
    g(s) &= \frac{f'(s)}{r^2} - \frac{f(s)}{r^3} \\
    h(s) &= \frac{f''(s)}{r^3} - \frac{3f'(s)}{r^4} + \frac{3f(s)}{r^5} \\
    q(s) &= -\frac{f''(s)}{r^2} + \frac{f'(s)}{r^3}
\end{align*}
With this notation, the \textbf{electric field} $\E$ is given by
\begin{align*}
    E_1(\x) &= 10 \left[ xzh - yq \right] \\
    E_2(\x) &= 10 \left[ yzh + xq \right] \\
    E_3(\x) &= 10 \left[ g + z^2h - \frac{f''(s)}{r} \right]
\end{align*}
and the \textbf{magnetic field} $\B$ is
\begin{align*}
    B_1(\x) &= 10 \left[ yq + xzh \right] \\
    B_2(\x) &= 10 \left[ -xq + yzh \right] \\
    B_3(\x) &= 10 \left[ -2g - (x^2 + y^2)h \right]
\end{align*}
Snapshots of the evolution of the magnitude of the electric field can be found in Figure~\ref{fig:radial_waves}.
\begin{figure}[htbp]
    \centering
    \includegraphics[width=0.7\textwidth]{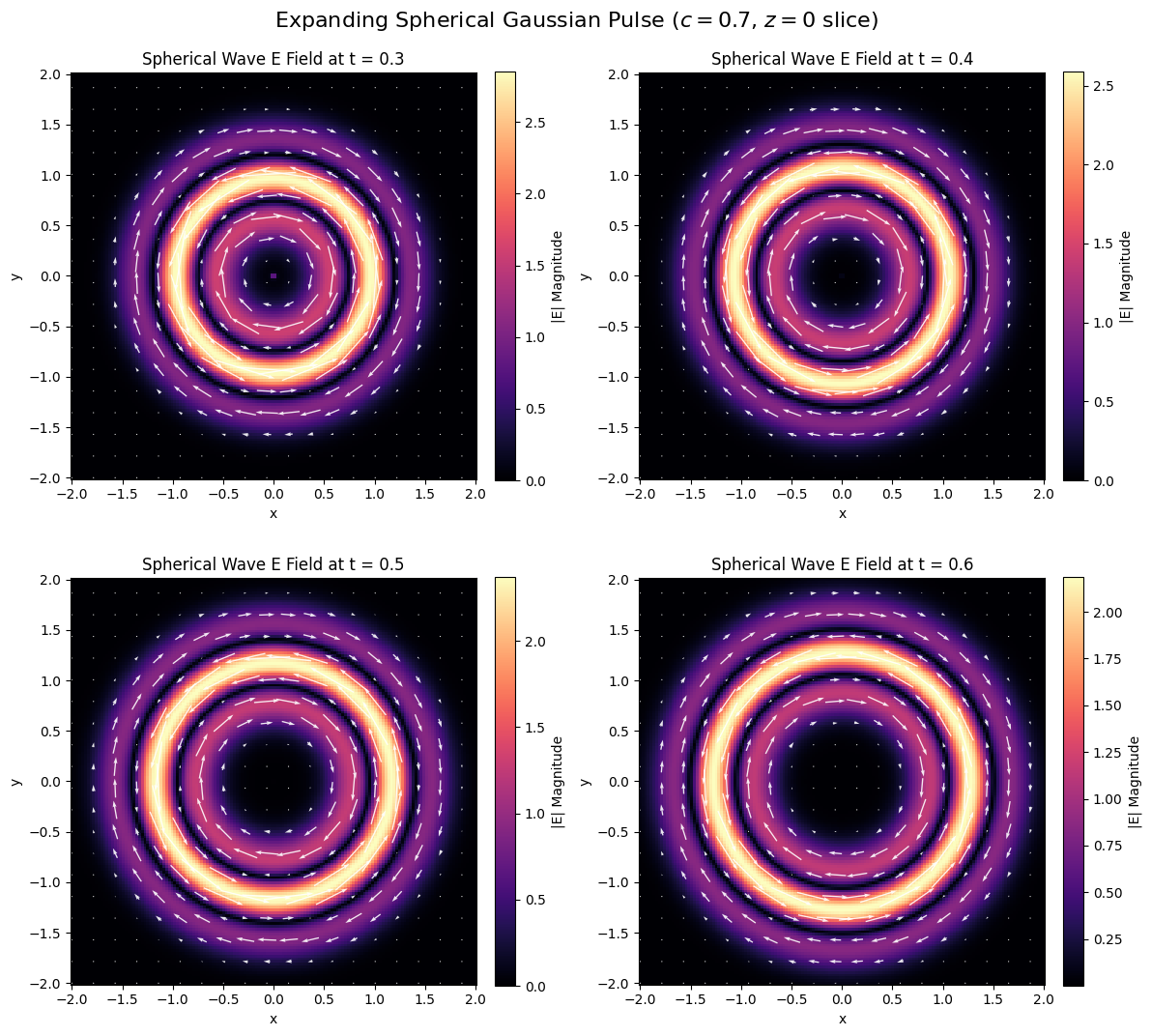}
    
    \caption{Time evolution of the magnitude of the electric field for the \textit{Radial Waves} solution.}
    \label{fig:radial_waves}
    \end{figure}

    \paragraph{Hopf Fibration.} We consider an electromagnetic knot or Hopfion, which is a smooth rational function (quotient of polynomials without singularities). Thus its nature is algebraic, which makes it completely different from the other three solutions we consider.
\begin{figure}[b]
    \centering
    \includegraphics[width=0.95\textwidth]{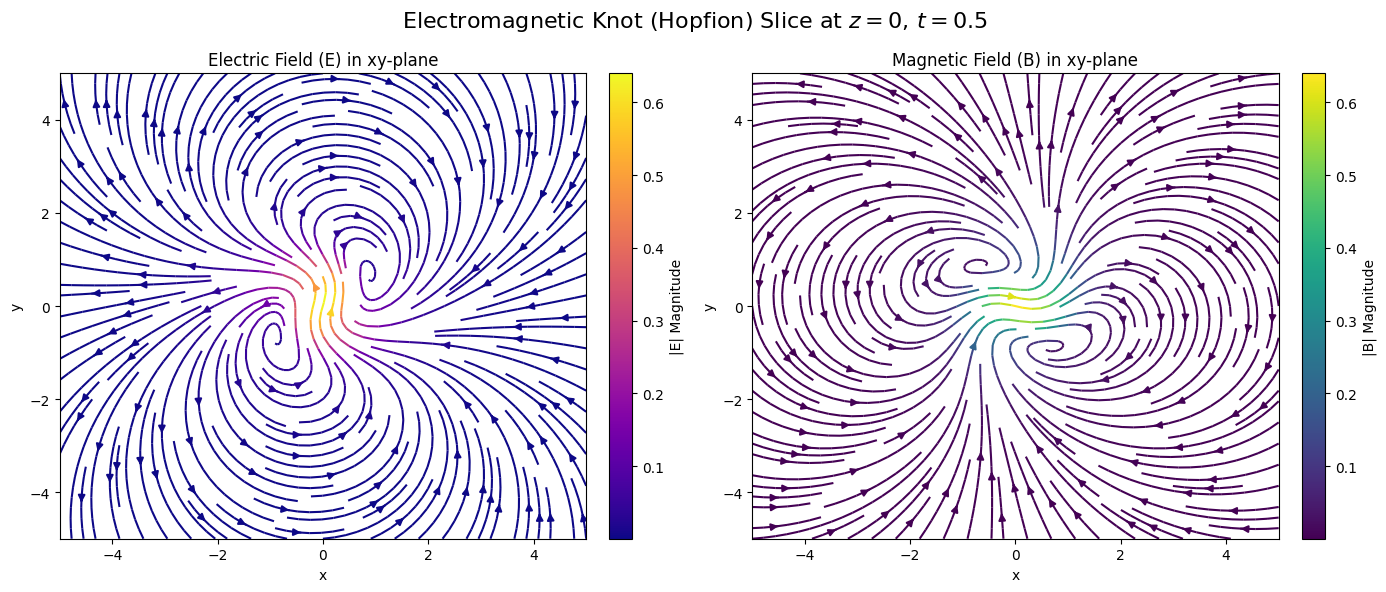}
    
    \caption{Planar streamlines of the \textit{Hopf Fibration} solution.}
    \label{fig:hopf_fibration}
    \end{figure}
    Consider the following auxiliary polynomial quantities:
\begin{align*}
       &A = 1 + x^2 + y^2 + z^2 - t^2,\qquad
 T_1 = A^3 - 12t^2A, \\
    &T_2 = 8t^3 - 6tA^2 \qquad
    D = (A^2 + 4t^2)^3.
\end{align*}
With this notation, the \textbf{electric field} $\E$ is given by
\begin{align*}
    E_1(\x) &=D^{-1} [{T_1 \left[ (t - z)^2 - 1 - x^2 + y^2 \right] - T_2 \left[ 2xy - 2(t - z) \right]}] \\
    E_2(\x) &= D^{-1}[{T_1 \left[ -2xy - 2(t - z) \right] - T_2 \left[ 1 - (t - z)^2 - x^2 + y^2 \right]}] \\
    E_3(\x) &= D^{-1}[{T_1 \left[ 2x(t - z) - 2y, \right] + T_2 \left[ 2x + 2y(t - z) \right]}]
\end{align*}
and the \textbf{magnetic field} $\B$ is
\begin{align*}
    B_1(\x) &= D^{-1}[{T_2 \left[ (t - z)^2 - 1 - x^2 + y^2 \right] + T_1 \left[ 2xy - 2(t - z) \right]}] \\
    B_2(\x) &= D^{-1}[{T_2 \left[ -2xy - 2(t - z) \right] + T_1 \left[ 1 - (t - z)^2 - x^2 + y^2 \right]}] \\
    B_3(\x) &= D^{-1}[{T_2 \left[ 2x(t - z) - 2y \right] - T_1 \left[ 2x + 2y(t - z) \right]}].
\end{align*}
Beautiful planar snapshots of the streamlines of the electric and magnetic fields can be found in Fig.~\ref{fig:hopf_fibration}.

\paragraph{Random Solution.}   For the last solution we add 100 plane waves similar to those used in the \textit{Plane Waves} solution. The direction and shifts of these plane waves are chosen randomly as we explain below. This ensures that the solutions we obtain are indistinguishable from an arbitrary solution of \eqref{eq:maxwell}.
\begin{figure}[h]
    \centering
    \includegraphics[width=0.7\textwidth]{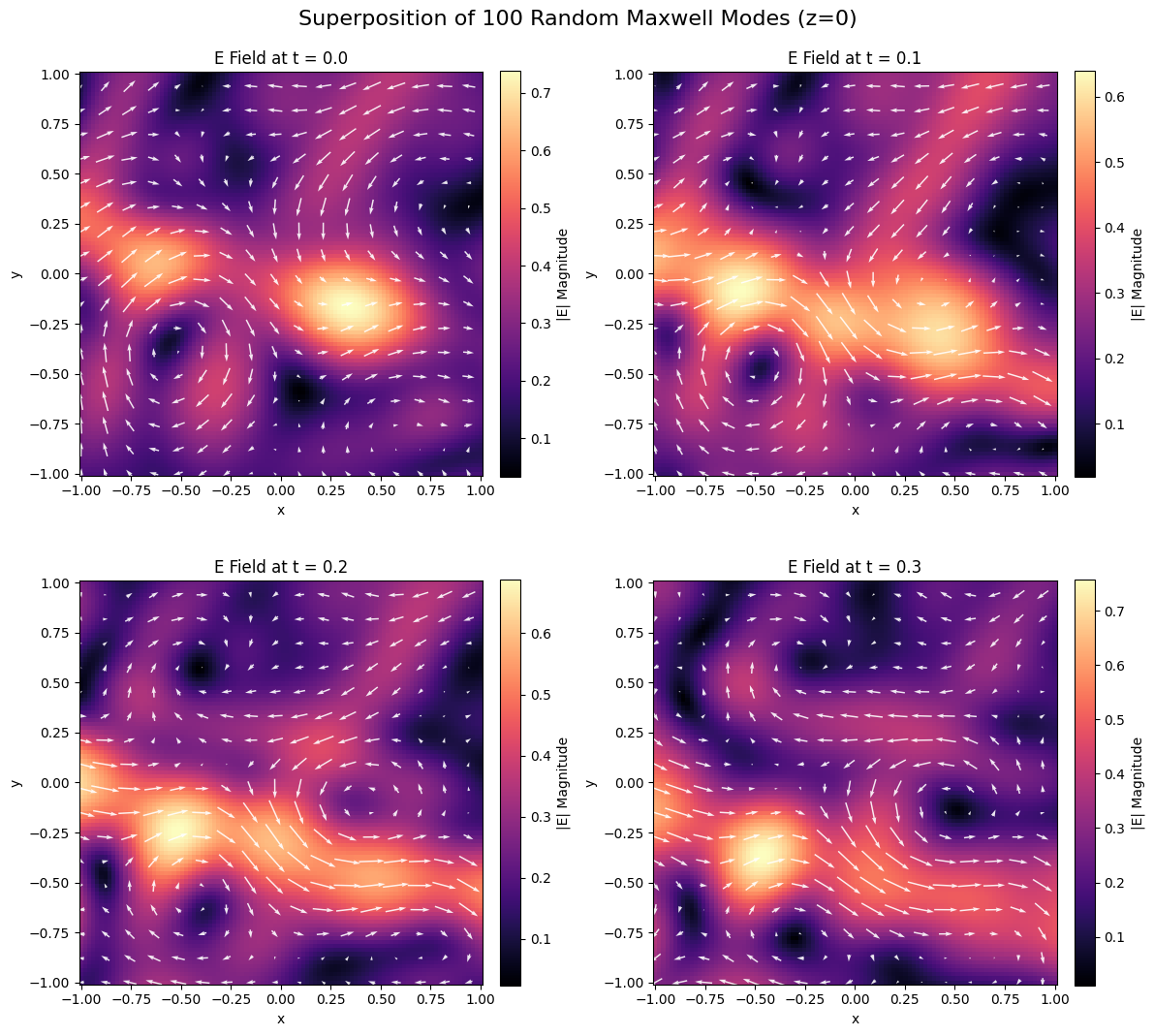}
    
    \caption{Time evolution of the magnitude of the electric field for the \textit{Random Solution}.}
    \label{fig:random_solution}
    \end{figure}

For each  $k \in \{1, 2, \dots, 100\}$, we sample  independently $z_{k1},z_{k2},z_{k3}\sim\mathcal N({0},{0.1})$ and $b_k\sim \mathcal N(0,1)$ and let $z_{k0}=\sqrt{z_{k1}^2+z_{k2}^2+z_{k3}^2}$. Let
\begin{align*}
    &s_k =  z_{k0}t + z_{k1}x + z_{k2}y + z_{k3}z + b_k\\
    &F(s) = -0.2 e^{-10(s - 0.3)^2} \left[ 1 - 20(s - 0.3)^2 \right]
\end{align*}
With this notation, the \textbf{electric field} $\E$ is given by
\begin{align*}
    E_1(\x) &= \sum_{k=1}^{100} \left[ z_{k1}z_{k3} - z_{k2}z_{k0} \right] F(s_k) \\
    E_2(\x) &= \sum_{k=1}^{100} \left[ z_{k2}z_{k3} + z_{k1}z_{k0} \right] F(s_k) \\
    E_3(\x) &= \sum_{k=1}^{100} \left[ z_{k3}^2 - z_{k0}^2 \right] F(s_k)
\end{align*}
and the \textbf{magnetic field} $\B$ is
\begin{align*}
    B_1(\x) &= \sum_{k=1}^{100} \left[ z_{k0} z_{k2} + z_{k1}z_{k3} \right] F(s_k) \\
    B_2(\x) &= \sum_{k=1}^{100} \left[ -z_{k0} z_{k1} + z_{k2}z_{k3} \right] F(s_k) \\
    B_3(\x) &= \sum_{k=1}^{100} \left[ -z_{k1}^2 - z_{k2}^2 \right] F(s_k).
\end{align*}
Snapshots of the evolution of the magnitude of the electric field can be found in Figure~\ref{fig:random_solution}.

\section{Proof of the Universal Approximation Theorem}\label{sec:proof_uat}
We first clarify what we mean by $L^2$-solutions of~\eqref{eq:maxwell}.  To do so, we define the spaces we are working with: Let $\Omega\subset\R^d$ be an open set; in our work $d=3,4$. We will  work with spaces of smooth functions 
$$
C^m(\Omega,\R^n)=\{u\colon\Omega\to\R^n\colon \partial^\alpha u\text{ uniformly continuous whenever }|\alpha|\leq m\}.
$$
To be clear, here $\alpha\in\mathbb{N}^d$ is a multi-index so
$$
\partial^\alpha=\partial_1^{\alpha_1}\partial_2^{\alpha_2}\ldots\partial_d^{\alpha_d}
$$
and $|\alpha|=\alpha_1+\alpha_2+\ldots+\alpha_d$. The spaces $C^m(\Omega,\R^n)$ are Banach spaces under the norms
$$
\|u\|_{C^m(\Omega,\R^n)}=\max_{|\alpha|\leq m,\x\in\Omega}|\partial^\alpha u(\x)|,
$$
where $|\cdot|$ denotes the $\ell^2$-norm on a finite dimensional inner product space (like $\R^n$).  No confusion will arise from using the same notation $|\cdot|$ for two types of length. The space $$C^\infty(\Omega,\R^n)=\bigcap_{m\geq 0}C^m(\Omega,\R^n)$$
is a Frech\'et space under the $C^m$-seminorms. Finally, we will work with the space $C_c^\infty(\Omega,\R^n)$ of compactly supported smooth functions, meaning that $u\in C_c^\infty(\Omega,\R^n)$ if $u\in C^\infty(\Omega,\R^n)$ and the set $\{\x\in\Omega\colon u(\x)\neq\boldsymbol{0}\}$ is compact (closed and bounded).

The typical space for error measurement is
$$
L^2(\Omega,\R^n)=\{u\colon\Omega\to\R^n\colon u\text{ measurable, }\int_\Omega|u(\x)|^2d \x<\infty\},
$$ 
We will also work with the Sobolev spaces $$H^m(\Omega,\R^n)=\{u\in L^2(\Omega,\R^n)\colon \partial^\alpha u\in L^2(\Omega,\R^n)\text{ whenever }|\alpha|\leq m\},$$
the space of functions with all derivatives of order at most $m$ belonging to $L^2$. In particular, $H^0=L^2$.  We say that $\partial^\alpha u\in L^2(\Omega)$ if there is a constant $C>0$ such that
$$
\int_\Omega u(\x)^\top \partial^\alpha v(\x)d \x\leq C\|v\|_{L^2(\Omega,\R^n)}\quad\text{for all }v\in C_c^\infty(\Omega,\R^n).
$$

The spaces $H^m(\Omega,\R^n)$ are Hilbert spaces with the inner product
$$
\langle u,v\rangle_{H^m(\Omega,\R^n)}=\sum_{|\alpha|\leq m}\int_{\Omega}\partial^\alpha u(\x)^\top\partial^\alpha  v(\x)d \x,
$$
which induces the norm
$$
\|u\|_{H^m(\Omega,\R^n)}=\sqrt{\sum_{|\alpha|\leq m}\int_\Omega|\partial^\alpha u(\x)|^2d \x}.
$$

An $L^2$-solution $(\E,\B)$ of~\eqref{eq:maxwell} in $\Omega\subset\R^4$ is simply a function $(\E,\B)\in L^2(\Omega,\R^6)$ that is a distributional solution. To define this here, write
$$
A(\partial)(\E,\B)=\begin{pmatrix}
    \partial_t\E-\nabla\times \B\\
    \partial_t\B+\nabla\times \E\\
    \nabla\cdot\E\\
    \nabla\cdot\B
\end{pmatrix},
$$
where $\partial=(\partial_t,\partial_x,\partial_y,\partial_z)$ and $\nabla=(\partial_x,\partial_y,\partial_z)^\top$. Then $A(\partial)(\E,\B)=\boldsymbol{0}$ in $\Omega$ if and only if
$$
\int_{\Omega} (\E(\x),\B(\x))^\top A(\partial)^\top \varphi(\x)d\x=0\quad\text{for all }\varphi\in C_c^\infty(\Omega,\R^8).
$$
In particular, if $(\E,\B)\in H^1(\Omega,\R^6)$, this distributional relation reduces by integration by parts to~\eqref{eq:model} expressed almost everywhere in $\Omega$.

We can now state our main theorem, extending Theorem~\ref{thm:spacetimeUAT} to various norms. We will only work with time-space domains of the form $(t_1,t_2)\times \Omega$ as it seems unphysical to allow for gaps in the time domain. However, the result below holds for any open set in $\R^4$. 
\begin{theorem}[Universal approximation for Maxwell's equations]
\label{thm:UAP}
Let $m\geq0$ be an integer, {$\sigma\in C^m(\R)$ be non-polynomial}, $t_1<t_2$, and $\Omega\subset\R^3$ be a bounded open   set.
Let $(\E_\gt,\B_\gt)$ be a $H^m$-solution of the homogeneous Maxwell equations~\eqref{eq:maxwell} on $\R^{4}$.
Then, for every $\varepsilon>0$, there exists a finite Maxwell network solution $(\E,\B)$ of the form~\eqref{eq:model} such that
\[
\|( \E_\gt, \B_\gt)-(\E,\B)\|_{H^m((t_1,t_2)\times\Omega)}<\varepsilon.
\]
If $m>1$, the same holds true with the $H^m$-norm replaced by the $C^{m-2}$-norm.
\end{theorem}
In particular, for $C^0$-approximations, convergence is uniform. Moreover, if we have a smooth solution $(\E,\B)\in C^\infty(\R^4,\R^6)$, we obtain an approximation in the Frech\'et topology on $C^\infty((t_1,t_2)\times\Omega)$.

The fact that functions in~\eqref{eq:model} satisfy Maxwell's equations~\eqref{eq:maxwell} follows by construction, since we use the solvepde command in Macaulay2 \cite{macaulay2}. We will prove this directly in the sequel.

Theorem~\ref{thm:UAP} follows from the next statement which incorporates boundary conditions. We will need to define the set
$$
\Omega_T=\{\x\in\R^3\colon \mathrm{dist}(\x,\Omega)<T\},
$$
which is the region of $\R$ where initial conditions uniquely determine the solution of~\eqref{eq:maxwell}  up to time $T$. This is an aspect of the so-called \textit{finite speed of propagation} property \cite[Theorem 3.12]{Lax2006}.
\begin{theorem}[Universal approximation for Maxwell's equations with initial conditions]
\label{thm:UAP_IC}
Let $m\geq0$ be an integer, {$\sigma\in C^m(\R)$ be non-polynomial}, $T>0$, and $\Omega\subset\R^3$ be a convex bounded open   set.
Let $\E_0,\B_0\in H^m(\Omega_T,\R^3)$ be divergence-free. Let $(\E_\gt,\B_\gt)$ be the unique solution in $H^m((0,T)\times\Omega,\R^6)$ of the homogeneous Maxwell equations~\eqref{eq:maxwell}.
Then, for every $\varepsilon>0$, there exists a finite Maxwell network solution $(\E,\B)$ of the form~\eqref{eq:model} such that
\begin{align*}
    \|(\E_0,\B_0)(0,\cdot)-(\E,\B)(0,\cdot))\|_{H^m(\Omega_T)}<\varepsilon,
\end{align*}
so that we can further ensure
\[
\|( \E_\gt, \B_\gt)(t,\cdot)-(\E,\B)(t,\cdot)\|_{H^m((0,T)\times\Omega)}<\varepsilon.
\]
If $m>1$, the second estimate holds true with the $H^m$-norm replaced by the $C^{m-2}$-norm. 
\end{theorem}
In particular, given $C^\infty$ initial conditions, we obtain approximations in the Frech\'et topology on $C^\infty((0,T)\times\Omega)$.

Before we prove this theorem, we need to ensure that our model makes sense, i.e.\ that~\eqref{eq:model} indeed generates solutions of Maxwell's system~\eqref{eq:maxwell}.

\begin{lemma}\label{lem:solution}
   ~\eqref{eq:model} gives solutions to Maxwell's equations~\eqref{eq:maxwell}.
\end{lemma}

\begin{proof}
    By linearity of the equations, it suffices to verify this claim for 
    $$
    (\E(\x),\B(\x))=\sigma(\x^\top\z)\p_i(\z)
    $$
    where 
    $$z_{0}=\pm\sqrt{z_{1}^2+z_{2}^2+z_{3}^2}$$ and
\begin{align*}
   \begin{split}
    &\p_1(\z)=(- z_1z_3, -z_2z_3,z_1^2+z_2^2,-z_0z_2,z_0z_1,0)\\
    &\p_2(\z)=(z_1z_2,-z_1^2-z_3^2,z_2z_3,-z_0z_3,0,z_0z_1).
    \end{split}
\end{align*}
We thus write $z^\pm:=(\pm|\xi|,\xi)$ for $\xi\in\R^3$, so that a direct computation gives
\[
\p_1(z^\pm)
=
\bigl(-\xi\times(\xi\times e_3),\ \mp |\xi|(\xi\times e_3)\bigr),
\qquad
\p_2(z^\pm)
=
\bigl(\xi\times(\xi\times e_2),\ \pm |\xi|(\xi\times e_2)\bigr),
\]

First let $i=1$. By the chain rule, we obtain
\begin{align*}
    &\partial_t\E-\nabla\times \B=\sigma'(\x^\top z^\pm)(\pm|\xi|(-\xi\times(\xi\times e_3))-\xi\times(\mp |\xi|(\xi\times e_3))=\boldsymbol{0}\\
    &\partial_t\B+\nabla\times \E=\sigma'(\x^\top z^\pm)(\pm|\xi|(\mp |\xi|(\xi\times e_3)+\xi\times(-\xi\times(\xi\times e_3)))=\boldsymbol{0}\\
    &\nabla\cdot \E=\sigma'(\x^\top z^\pm)(-\xi\times(\xi\times e_3))\cdot\xi=\boldsymbol{0}\\
    &\nabla\cdot \B=\sigma'(\x^\top z^\pm)(\mp |\xi|(\xi\times e_3)\cdot\xi=\boldsymbol{0},
\end{align*}
where all calculations follow by elementary vector calculus identities. The case $i=2$ is similar.
\end{proof}

\begin{proof}[Proof of Theorem~\ref{thm:UAP} assuming Theorem~\ref{thm:UAP_IC}]
    If $(\E_\gt,\B_\gt)\in H^m(\R^4)$, then $(\E_\gt(t,\cdot),\B_\gt(t,\cdot))\in H^m(\R^3)$ for almost every $t\in \R$. This follows by Fubini's theorem. Pick such a $t_0\leq t_1$. By translation invariance of linear pde with constant coefficients, we can assume that $t=0$ and $t_2=T$ in the statement of Theorem~\ref{thm:UAP_IC}. By replacing $\Omega$ with its convex hull, we can assume that $\Omega$ is convex. We also let $\E_0=\E_\gt(0,\cdot)$ and $\B_0=\B_\gt(0,\cdot)$. Apply Theorem~\ref{thm:UAP_IC} to conclude.
\end{proof}
Of course, the difficulty is moved to the:
\begin{proof}[Proof of Theorem~\ref{thm:UAP_IC}]
    We first establish the approximation of $\E_0,\B_0$ in $H^m$ using activation function $\sigma=\cos$. Then we extend the approximation to general $\sigma$. Finally, we recall estimates for~\eqref{eq:maxwell} that enable us to control the spacetime error by the error at $t=0$.

    \textit{Step 1. We can assume $\E_0,\B_0\in L^2(\R^3)$ and $\E_0=\boldsymbol{0}=\B_0$ outside a cube $Q\subset\R^3$, without loss of generality:}\\
    By \cite{KatoMitreaPonceTaylor2000}, there exist $\tilde\E_0,\tilde\B_0\in L^2(\R^3)$ that are both divergence-free in $\R^3$ and each equal $\E_0,\B_0$ in $\Omega_T$. By Poincar\'e's lemma, $\tilde\E_0=\nabla\times v_1,\,\tilde\B_0=\nabla\times v_2$ for some $v_1,\,v_2\in H^1(\R^3)$. Consider $\rho\in C_c^\infty(Q)$ such that $\rho=1$ in $\Omega_T$ and $Q\Supset \Omega_T$ is a sufficiently large cube. Then $e_0=\nabla\times(\rho v_1)$, $b_0=\nabla\times(\rho v_2)$ are divergence-free in $\R^3$, compactly supported inside $Q$, equal $\E_0,\B_0$ respectively in $\Omega_T$. 
    Let $(e,b)$ be the unique solution in $L^2((0,T)\times\R^3)$ of~\eqref{eq:maxwell} with initial conditions $e_0,b_0$. 
    By the finite speed of propagation property (see also \textit{Step 4} below), we have that $e=\E_\gt$ and $b=\B_\gt$ in $(0,T)\times \Omega$.

    \textit{Step 2. $\E_0,\B_0$ can be approximated in {$H^m$} by Fourier sums with ``good'' frequencies:}\\
    By the translation-, dilation-, and rotation-invariance of~\eqref{eq:maxwell}, we can assume the cube $Q=(0,2\pi)^3$, so that we can write the Fourier series
    \begin{align*}
        \E_0(x,y,z)=\sum_{{\boldsymbol{k}}\in \mathbb Z^3} c_{{\boldsymbol{k}}}\exp(\sqrt{-1}(xk_1+yk_2+zk_3))
    \end{align*}
where $c_{{\boldsymbol{k}}}\in \mathbb C^3$ are complex coefficients, with a similar formula for $\B_0$.
    
    We next use the fact that $\E_0$ is divergence-free to conclude, first, that $c_{\mathbf 0}=\boldsymbol{0}$. This is so since
    $$
    (2\pi)^3c_{\mathbf{0}}=\int_Q \E_0{d}V=\int_Q \nabla(x,y,z)\E_0{d}V=-\int_Q (x,y,z)\nabla \cdot \E_0{d}V=\mathbf{0}.
    $$
    We moreover have
    \begin{align*}
        0=\nabla\cdot \E_0(x,y,z)=\sqrt{-1}\sum_{\mathbf0\neq{\boldsymbol{k}}\in \mathbb Z^3} c_{{\boldsymbol{k}}}\cdot {\boldsymbol{k}}\exp(-\sqrt{-1}(xk_1+yk_2+zk_3)),
    \end{align*}
so that the scalar product $c_{{\boldsymbol{k}}}\cdot{\boldsymbol{k}}=0$ for all $\mathbf0\neq{\boldsymbol{k}}\in \mathbb{Z}^3$ by linear independence over $\R$ of sine and cosine waves. In particular, if $c_{\boldsymbol{k}}=a_{\boldsymbol{k}}+\sqrt{-1}b_{\boldsymbol{k}}$ for $a_{\boldsymbol{k}},b_{\boldsymbol{k}}\in\mathbb R^3$, then $a_{\boldsymbol{k}},b_{\boldsymbol{k}}\in {\boldsymbol{k}}^\perp$ for any $\mathbf0\neq{\boldsymbol{k}}\in \mathbb{Z}^3$. After taking real part we are left with
\begin{align*}
        \E_0(x,y,z)=\sum_{\mathbf0\neq{\boldsymbol{k}}\in \mathbb Z^3} a_{{\boldsymbol{k}}}\cos(xk_1+yk_2+zk_3)-b_{{\boldsymbol{k}}}\sin(xk_1+yk_2+zk_3),
    \end{align*}
and $\sin$ is a shift of $\cos$ by $\pi/2$. We have a similar formula for $B_0$,
\begin{align*}
        \B_0(x,y,z)=\sum_{\mathbf0\neq{\boldsymbol{k}}\in \mathbb Z^3} \tilde a_{{\boldsymbol{k}}}\cos(xk_1+yk_2+zk_3)-\tilde b_{{\boldsymbol{k}}}\sin(xk_1+yk_2+zk_3),
    \end{align*}
with $\tilde a_{\boldsymbol{k}},\tilde b_{\boldsymbol{k}}\in {\boldsymbol{k}}^\perp$. We denote by $\E_{0j},\B_{0j}$ the partial sums of $\E_0,\B_0$ above, truncated at $|{\boldsymbol{k}}|\leq j$. Thus $\E_{0j},\B_{0j}$ are smooth divergence-free fields, which converge respectively to $\E_0,\B_0$ in $H^m(Q,\R^3)$ as $j\to\infty$.

We next approximate $\E_{0j},\B_{0j}$ with similar trigonometric polynomials such that no frequency used has $k_1=0$.\footnote{The reason why this is necessary is far from transparent and will become clear in \textit{Step 3}. In short, the vectors $\p_i(\z)\in\R^6$ from~\eqref{eq:noetherian_multipliers} contain too many zeroes at $\z\in\R^4$ with $\z_1=0$.}
To this end, we will show that any derivative $\partial^\alpha$, $|\alpha|\leq m$ of the term $a\cos(k_2y+k_3z)$ with $a\in\R^3$ and $k_2^2+k_3^2\leq j^2$ can be approximated uniformly with a good frequency. Note that  $a_2k_2+a_3k_3=0$ and one of $k_2,k_3$ must be non-zero. Suppose without loss of generality that $k_2\neq 0$. Note that for any $\delta\in(0,1)$
$$
  \nabla \cdot[  (a_1,a_2-a_1\delta/k_2,a_3)^\top\cos(\delta x+ k_2y+k_3z)]=0.
$$
We aim to estimate for $\alpha\in\mathbb N^3$ with $|\alpha|\leq m$ the difference
\begin{align*}
    \textbf{I}_\alpha=|\partial^\alpha[(a_1,a_2-a_1\delta/k_2,a_3)^\top\cos(\delta x+ k_2y+k_3z)-a\cos( k_2y+k_3z)] |.
\end{align*}
To do so, we introduce a bit of notation. We write $f_k$ for the $k$th derivative of $f$. We write $v_1=(\delta,k_2,k_3)^\top$, $v_2=(0,k_2,k_3)^\top$, $\vec{x}=(x,y,z)^\top$, and $\tilde a=(a_1,a_2-a_1\delta/k_2,a_3)^\top$. We then have by chain rule
\begin{align*}
\mathbf{I}_\alpha &=|\partial^\alpha[\tilde a\cos(v_1^\top \vec{x})-a\cos( v_2^\top \vec{x})] |\\
&=|v_1^\alpha \tilde a \cos_{|\alpha|}(v_1^\top \vec{x})-v_2^\alpha a\cos_{|\alpha|}(v_2^\top \vec{x})|\\
&\leq|v_1^\alpha \tilde a \cos_{|\alpha|}(v_1^\top \vec{x})-v_1^\alpha  a \cos_{|\alpha|}(v_1^\top \vec{x})|+|v_1^\alpha  a \cos_{|\alpha|}(v_1^\top \vec{x})-v_2^\alpha  a \cos_{|\alpha|}(v_1^\top \vec{x})|\\
&\phantom{leq}+|v_2^\alpha  a \cos_{|\alpha|}(v_1^\top \vec{x})-v_2^\alpha a\cos_{|\alpha|}(v_2^\top \vec{x})|\\
&\leq |v_1^\alpha||\tilde a-a|+|v_1-v_2||a|+|v_2^\alpha||a||\delta x|\\
&\leq j^m|a|\delta+\delta |a|+2\pi\delta |a||j|^m\\
&\leq 10\delta |a|j^m
\end{align*}
where we used that $|k_2|\geq1$ and the fact that $\cos$ is 1-Lipschitz. It follows that
$$
\|\tilde a\cos(v_1^\top \cdot)-a\cos( v_2^\top \cdot)\|_{C^m(Q)}\leq C_0(m)\delta|a|j^m;
$$
In particular,  
 replacing the  terms in the \textit{finite} expansions of $\E_{0j},\B_{0j}$ with frequencies $\boldsymbol{k}$ with $k_1=0$  as above by choosing $\delta $  small enough, we can obtain
$$
\|(\E_{0j},\B_{0j})-(\tilde \E_{0j},\tilde \B_{0j})\|_{C^m(Q)}\leq 1/j,
$$
where
$$
(\tilde \E_{0j},\tilde \B_{0j})^\top=\sum_{\xi\in F_j}A_{\xi}\cos(\xi_1 x+\xi_2y+\xi_3z)+B_{\xi}\sin(\xi_1 x+\xi_2y+\xi_3z)
$$
with $\xi_1\neq 0$ for all $\xi$ in a finite set $F_j\subset\R^3\setminus\{\boldsymbol{0}\}$ and $A_{\xi},B_{\xi}\in \ker R({\xi})$, where
$$
R(\xi)=\left(\begin{matrix}
    \xi_1&\xi_2&\xi_3&0&0&0\\
    0&0&0&\xi_1&\xi_2&\xi_3
\end{matrix}\right).
$$
Then $(\tilde \E_{0j},\tilde \B_{0j})$ give the desired finite approximations of $(\E_0,\B_0)$ in $H^m(Q,\R^6)$.

 \textit{Step 3. $(\E_0,\B_0)$ can be approximated in $H^m$ by~\eqref{eq:model} with $\sigma=\cos$ at $t=0$:}\\
To achieve this, we will express $(\tilde \E_{0j},\tilde \B_{0j})$ in terms of 
  $$
    P(\xi)=(p_1(|\xi|,\xi)^\top,\,p_2(|\xi|,\xi)^\top,\,p_1(-|\xi|,\xi)^\top,\,p_2(-|\xi|,\xi)^\top)\in\R^{6\times 4}.
    $$
    with $p_i$ as in~\eqref{eq:noetherian_multipliers}.

    We claim that $\mathrm{im\,} P(\xi)=\ker R(\xi)$ if $\xi_1\neq0$.  To this end, we note that the inclusion $\subset$ follows from a simple computation. Once we prove that $\ker P({\xi})=\{\boldsymbol0\}$, the rank-nullity theorem implies that $\mathrm{rank\,}P({\xi})=4$, which clearly equals the dimension of $\ker R({\xi})$, yielding the claim. To show that $\ker P({\xi})=\{\boldsymbol0\}$ for $\xi_1\neq0$, we simply write
    $$
   \alpha_1 p_1(|{\xi}|,{\xi})+\alpha_2p_2(|{\xi}|,{\xi})+\alpha_3p_1(-|{\xi}|,{\xi})+\alpha_4p_2(-|{\xi}|,{\xi})=\mathbf 0.
    $$
    Equation~\eqref{eq:noetherian_multipliers} implies $\alpha_1=\alpha_3$, $\alpha_2=\alpha_4$. If the $(1,1),(3,2)$-minor is nonzero, we are done. Else, $\xi_2=0$ and then the $(2,1),(3,2)$ minor is nonzero, which proves the claim.


    Let 
    \begin{align*}
            v_j(x,y,z)=\sum_{|{\xi}|\in F_j}&[P({\xi})^\top P({\xi})]^{-1}P({\xi})^\top A_{{\xi}}\cos(x\xi_1+y\xi_2+z\xi_3)\\
            &+[P({\xi})^\top P({\xi})]^{-1}P({\xi})^\top B_{{\xi}}\sin(x\xi_1+y\xi_2+z\xi_3),
    \end{align*}
    where the left-inverse is actually an inverse because $\ker R(\xi)=\mathrm{im\,}P(\xi)$ for $\xi_1\neq 0$. This implies that 
    \begin{align}\label{eq:sum}
    \begin{split}
    \sum_{|{\xi}|\in F_j}&P({\xi})[P({\xi})^\top P({\xi})]^{-1}P({\xi})^\top A_{{\xi}}\cos(x\xi_1+y\xi_2+z\xi_3)\\
            &+P({\xi})[P({\xi})^\top P({\xi})]^{-1}P({\xi})^\top B_{\xi}\sin(x\xi_1+y\xi_2+z\xi_3)=(\tilde \E_{0j},\tilde \B_{0j})^\top(x,y,z),
            \end{split}
    \end{align}
    which means that indeed arbitrary divergence-free initial conditions can be approximated in $H^m$ using~\eqref{eq:model} with activation function $\sigma=\cos$. 

       \textit{Step 4. General activation function for $H^m$-approximation at $t=0$:}\\
     Because the sum in~\eqref{eq:sum} is finite, it suffices to approximate a single term
     $$
     v(\vec{x})=P(\xi)A \cos(\xi^\top\vec{x}+b)
     $$
     in $H^m(Q)$ using an arbitrary activation $\sigma$. In particular, by our assumption on $\sigma$ not being a polynomial, we can use \cite{hornik1990universal} to find parameters $w_j,\,z_j,\,b_j\in\R$ such that
     $$
     \left\|\cos-\sum_{j=1}^Ww_j\sigma(z_j\cdot+b_j)\right\|_{H^m(-4\pi|\xi|+b,4\pi|\xi|+b)}<\delta.
     $$
     Our approximation for $v$ will then be 
     $$\tilde v(\vec{x})=\sum_{j=1}^Ww_jP(\xi)A\sigma(z_j\xi^\top\vec x+z_jb+b_j)=\sum_{j=1}^WP(z_j\xi)(w_jz_j^{-2}A)\sigma(z_j\xi^\top\vec x+\underbrace{z_jb+b_j}_{B_j}),$$ 
so $\tilde v$ is in the form~\eqref{eq:model} at $t=0$. We will next show that we can control the $\alpha$-derivative of the error in approximating $v$ with the error of approximating the $\alpha$-derivative of $\cos$. Let $\alpha\in\mathbb N^3$.
\begin{align*}
    &\|\partial^\alpha(v-\tilde v)\|_{L^2(Q)}^2\\
    &=\int_Q \left|\xi^\alpha P(\xi)A \left(\cos_{|\alpha|}(\xi^\top \vec{x}+b)-\sum_{j=1}^Ww_jz_j^{|\alpha|}\sigma_{|\alpha|}(z_j\xi^\top \vec{x}+z_jb+b_j)\right)\right|^2d V\\
    &\leq |\xi^\alpha P(\xi)A|^2\int_{(-4\pi,4\pi)^3}\left|\cos_{|\alpha|}(|\xi|t_1+b)-\sum_{j=1}^Ww_jz_j^{|\alpha|}\sigma_{|\alpha|}(z_j(|\xi|t_1+ b)+b_j)\right|^2dt_1dt_2dt_3\\
    &\leq 64\pi^2|\xi^\alpha P(\xi)A|^2|\xi|^{-1}\int_{-4\pi|\xi|+b}^{4\pi|\xi|+b}\left|\cos_{|\alpha|}(s)-\sum_{j=1}^Ww_jz_j^{|\alpha|}\sigma_{|\alpha|}(z_js+b_j)\right|^2ds\\
    &\leq 64\pi^2|\xi^\alpha P(\xi)A|^2|\xi|^{-1}\delta^2.
    \end{align*}
    where the first change of variable is a rotation making $\xi$ the direction of the $t_1$ axis ($4\pi$ is a generous bound to ensure that the rotation of $Q$ is contained in $(-4\pi,4\pi)^3$) and the last inequality follows from the approximation of cosine. By summing these estimates of $\partial^\alpha(v-\tilde v)$ for $|\alpha|\leq m$, we can conclude that $\|v-\tilde v\|_{H^m(Q)}$ can be made arbitrarily small.

    \textit{Step 5. Dependence of solutions of Maxwell's equations  on initial conditions in $H^m$ or $C^m$:}\\
    Let $(\E,\B)\in H^m$ be a solution of~\eqref{eq:maxwell}. By direct calculation we obtain the identity
    $$
    \partial_t(|\E|^2+|\B|^2)=2\nabla \cdot (\B\times  \E).
    $$
    By integrating in space we obtain
    $$
    \partial_t\int_{\R^3}|\E|^2+|\B|^2{d}V=2\int_{\R^3}\nabla\cdot(\B\times \E){d}V=0,
    $$
    where the latter equality follows from the divergence theorem and the fact that $\B(t,\cdot),\E(t,\cdot)$ are compactly supported for all $t>0$ (another consequence of the finite speed of propagation).

    It follows that the $L^2$-energy on the right hand side is constant in time, so 
    \begin{align*}
         \int_{\R^3}|\E|^2+|\B|^2{d}V= \int_{\R^3}|\E(0,\cdot)|^2+|\B(0,\cdot)|^2{d}V=\int_{\R^3}|\E_0|^2+|\B_0|^2{d}V.
    \end{align*}
    In particular,
    \begin{align}\label{eq:continuous_dependence_L2}
        \sup_{t\in(0,T)}\int_{\R^3}|\E(t,\cdot)|^2+|\B(t,\cdot)|^2{d}V\leq \int_{\R^3}|\E_0|^2+|\B_0|^2{d}V.
    \end{align}

     Since all derivatives of $(\E,\B)$ are also solutions of~\eqref{eq:maxwell}, we can consider $\alpha\in\mathbb  N^3$ with $|\alpha|\leq m$ and only spatial derivatives $\partial=(\partial_x,\partial_y,\partial_z)$ to obtain directly from~\eqref{eq:continuous_dependence_L2} that
     \begin{align*}
        \sup_{t\in(0,T)}\int_{\R^3}|\partial^\alpha\E(t,\cdot)|^2+|\partial^\alpha\B(t,\cdot)|^2{d}V\leq \int_{\R^3}|\partial^\alpha\E_0|^2+|\partial^\alpha\B_0|^2{d}V,
    \end{align*}
    which implies
    \begin{align}\label{eq:continuous_dependence_Hm}
        \sup_{t\in(0,T)}\|(\E,\B)(t,\cdot)\|_{H^m(\Omega)}\leq \|(\E_0,\B_0)\|_{H^m(\Omega_T)}.
    \end{align}
We next incorporate time derivatives. Let $a\in\mathbb N$ and $\alpha\in\mathbb N^{3}$ be such that $a+|\alpha|\leq m$. We claim that $\partial_t^a\partial^\alpha (\E,\B)$ can be expressed only in terms of spatial derivatives of $(\E,\B)$. To this end, we note that by direct calculation we have for integers $j\geq 0$
$$
\partial_t^{2j}(\E,\B)=\Delta^j(\E,\B)\quad\text{so}\quad\partial_t^{2j+1}(\E,\B)=\Delta(\nabla\times \B,-\nabla\times \E),
$$
where $\Delta=\partial_x^2+\partial_y^2+\partial_z^2$. Therefore we can express 
\begin{align*}
    \partial_t^a\partial^\alpha (\E,\B)=A_{a,\alpha}(\partial)(\E,\B),
\end{align*}
where $A_{a,\alpha}$ is a linear partial differential operator in spatial variables only with constant coefficients and having order $a+|\alpha|\leq m$. In particular we can estimate
\begin{align*}
    \|\partial_t^a\partial^\alpha(\E,\B)\|_{L^2((0,T)\times\Omega)}&=\|A_{a,\alpha}(\partial)(\E,\B)\|_{L^2((0,T)\times\Omega)}\\
    &\leq \sqrt{T}\sup_{t\in(0,T)}\|A_{a,\alpha}(\partial)(\E,\B)(t,\cdot)\|_{L^2(\Omega)}\\
    &\leq C_1(m)\sqrt{T}\sup_{t\in(0,T)}\|(\E,\B)(t,\cdot)\|_{H^m(\Omega)}\\
    &\leq C_1(m)\sqrt{T}\|(\E_0,\B_0)\|_{H^m(\Omega_T)}.
\end{align*}
where $C(m)>0$ is a constant depending on $m$ only and the last inequality follows from~\eqref{eq:continuous_dependence_Hm}. Since $a+|\alpha|\leq m$ we can infer that
\begin{align}\label{eq:Hm_estimate}
    \|(\E,\B)\|_{H^m((0,T)\times\Omega)}\leq C_2(m)\sqrt{T}\|(\E_0,\B_0)\|_{H^m(\Omega_T)}
\end{align}
We derive a similar estimate on the $C^m$-scale. We write for $a+|\alpha|\leq m$
\begin{align*}
    \|\partial_t^a\partial^\alpha(\E,\B)\|_{C^0((0,T)\times\Omega)}&=\|A_{a,\alpha}(\partial)(\E,\B)\|_{C^0((0,T)\times\Omega)}\\
    &=\sup_{t\in(0,T)}\|A_{a,\alpha}(\partial)(\E,\B)(t,\cdot)\|_{C^0(\Omega)}\\
    &\leq C_3(m)\sup_{t\in(0,T)}\|(\E,\B)(t,\cdot)\|_{C^m(\Omega)}\\
    &\leq C_1(m,\Omega)\sup_{t\in(0,T)}\|(\E,\B)(t,\cdot)\|_{H^{m+2}(\Omega)}\\
    &\leq C_1(m,\Omega)\|(\E_0,\B_0)\|_{H^{m+2}(\Omega_T)}.
\end{align*}
where the penultimate inequality follows from the Morrey--Sobolev embedding and the last inequality follows from~\eqref{eq:continuous_dependence_Hm}. Finally, we get 
\begin{align}\label{eq:Cm_estimate}
    \|(\E,\B)\|_{C^m((0,T)\times\Omega)}\leq C_2(m,\Omega)\|(\E_0,\B_0)\|_{H^{m+2}(\Omega_T)}.
\end{align}

    \textit{Step 6. Conclusion in $H^m$ spaces:}\\
    We now collect most of the facts above. Let $\varepsilon>0$, $m\geq 0$, and $\E_0,\B_0\in H^m(\Omega_T)$ be divergence-free. Let $\E_\gt,\B_\gt\in H^m((0,T)\times \Omega)$ be the unique solution of Maxwell equations~\eqref{eq:maxwell} with initial conditions $\E_0,\B_0$.
    
    According to \textit{Steps 1-4}, there exists a finite neural network~\eqref{eq:model} which we name here $(\E,\B)$ such that
    \begin{align}\label{eq:IC_approximation_Hm}
        \|(\E_0,\B_0)-(\E,\B)(0,\cdot)\|_{H^m(\Omega)}<\varepsilon.
    \end{align}
    Since by \textit{Step 5}, $(\E,\B)$ is a solution to~\eqref{eq:maxwell}, so is its difference to $(\E_\gt,\B_\gt)$ by linearity. We can thus apply estimate~\eqref{eq:Hm_estimate} from Lemma~\ref{lem:solution} to obtain 
    $$
    \|(\E_\gt,\B_\gt)-(\E,\B)\|_{H^m((0,T)\times \Omega)}<C_1(m)\sqrt T\varepsilon.
    $$
Perhaps by making $\varepsilon$ smaller, we obtain the claim.

    \textit{Step 7. Conclusion in $C^m$ spaces:}\\
Now let $m>1$ and assume that the initial condition is in $C^m(\Omega_T)\subset H^m(\Omega_T)$. Building on \textit{Step 6}, particularly~\eqref{eq:IC_approximation_Hm}, we obtain from~\eqref{eq:Cm_estimate} that 
\begin{align*}
     \|(\E_\gt,\B_\gt)-(\E,\B)\|_{C^{m-2}((0,T)\times \Omega)}<C_2(m,\Omega)\varepsilon.
\end{align*}
This concludes the proof of our main result.
\end{proof}

\newpage

\end{document}